\documentclass[10pt,fullpage,letterpaper]{article}

\usepackage[utf8]{inputenc} 
\usepackage[T1]{fontenc}    
\usepackage{hyperref}       
\usepackage{url}            
\usepackage{booktabs}       
\usepackage{amsfonts}       
\usepackage{nicefrac}       
\usepackage{microtype}      
\usepackage{xcolor}         
\usepackage{algorithm}
\usepackage{algorithmic}
\usepackage{enumerate}
\usepackage{amsmath, amsfonts, amssymb}
\usepackage{amsthm}
\usepackage{mathrsfs}
\usepackage{bm}
\usepackage{graphicx}
\usepackage{alltt}
\usepackage{tikz}
\usepackage{float}
\usepackage{booktabs}
\usepackage{hyperref}
\usepackage{array}
\usepackage{dsfont}
\usepackage{nicefrac}
\usepackage{comment}
\usepackage{anyfontsize}
\usepackage{longtable}
\usepackage{natbib}
\usepackage{caption}
\usepackage[symbol]{footmisc}
\allowdisplaybreaks
\usepackage{mathtools}

\DeclareMathOperator*{\argmin}{argmin}
\DeclareMathOperator*{\argmax}{argmax}
\DeclareMathOperator{\proj}{proj}

\newcommand{\Sp}[1]{\left(#1\right)}
\newcommand{\Mp}[1]{\left[#1\right]}
\newcommand{\Bp}[1]{\left\{#1\right\}}
\newcommand{\abs}[1]{\left|#1\right|}
\newcommand{\Norm}[1]{\left\|#1\right\|}

\newcommand{\inner}[1]{\left\langle#1\right\rangle}
\newcommand{\ov}{\overline{V}}
\newcommand{\uv}{\underline{V}}
\newcommand{\oq}{\overline{Q}}
\newcommand{\uq}{\underline{Q}}
\newcommand{\K}{\mathcal{K}}
\newcommand{\A}{\mathcal{A}}
\newcommand{\B}{\mathcal{B}}
\newcommand{\D}{\mathcal{D}}
\newcommand{\E}{\mathbb{E}}
\newcommand{\G}{\mathcal{G}}
\renewcommand{\P}{\mathbb{P}}
\renewcommand{\S}{\mathcal{S}}
\newcommand{\um}{\underline{\mu}}
\newcommand{\un}{\underline{\nu}}
\newcommand{\on}{\overline{\nu}}
\newcommand{\om}{\overline{\mu}}

\newcommand{\1}{\mathbf{1}}
\renewcommand{\a}{\mathbf{a}}
\newcommand{\R}{\mathbb{R}}
\newcommand{\N}{\mathcal{N}}
\newcommand{\C}{\mathcal{C}}
\renewcommand{\r}{\mathbf{r}}
\newcommand{\algonanametwo}{\textsf{SBMM}}
\newcommand{\algonamemulti}{\textsf{SBSM}}

\newtheorem{theorem}{Theorem}
\newtheorem{lemma}{Lemma}
\newtheorem{definition}{Definition}
\newtheorem{assumption}{Assumption}
\newtheorem{proposition}{Proposition}
\newtheorem{corollary}{Corollary}

\usepackage{color}
\definecolor{ForestGreen}{rgb}{0.1333,0.5451,0.1333}
\hypersetup{colorlinks,
	linkcolor=ForestGreen,
	citecolor=ForestGreen,
	urlcolor=black,
	linktocpage,
	plainpages=false}

\author{Qiwen Cui\footnote{\textbf{University of Washington. Email: \url{qwcui@cs.washington.edu}}}  \and Simon S. Du\footnote{\textbf{University of Washington. Email: \url{ssdu@cs.washington.edu}}} }
\date{}

\title{Provably Efficient Offline Multi-agent Reinforcement Learning via Strategy-wise Bonus}

\begin{document}

\maketitle

\begin{abstract}


This paper considers offline multi-agent reinforcement learning.
We propose the strategy-wise concentration principle which directly builds a confidence interval for the joint strategy, in contrast to the point-wise concentration principle that builds a confidence interval for each point in the joint action space.
For two-player zero-sum Markov games, by exploiting the convexity of the strategy-wise bonus, we propose a computationally efficient algorithm whose sample complexity enjoys a better dependency on the number of actions than the prior methods based on the point-wise bonus. Furthermore, for offline multi-agent general-sum Markov games,  based on the strategy-wise bonus and a novel surrogate function, we give the first algorithm whose sample complexity only scales $\sum_{i=1}^mA_i$ where $A_i$ is the action size of the $i$-th player and $m$ is the number of players.
In sharp contrast, the sample complexity of methods based on the point-wise bonus would scale with the size of the joint action space $\Pi_{i=1}^m A_i$ due to the curse of multiagents.
Lastly, all of our algorithms can naturally take a pre-specified strategy class $\Pi$ as input and output a strategy that is close to the best strategy in $\Pi$. In this setting, the sample complexity only scales with $\log |\Pi|$ instead of $\sum_{i=1}^mA_i$.

\end{abstract}

\section{Introduction}

Multi-agent reinforcement learning (MARL)  is about decision making in a multi-agent system under uncertainty, which has achieved significant success in solving a wide range of tasks such as GO \citep{silver2017mastering}, Poker \citep{brown2019superhuman} and autonomous deriving \citep{shalev2016safe}. 
One standard setting in MARL is multi-player general-sum Markov games where each player deploys a policy to maximize its own total reward while the evolution of the environment depends on the policies of all the players \citep{zhang2021multi}. During the learning process, each player needs to identify the environment dynamics as well as compete/cooperate with other agents.  

One emerging  subarea is offline MARL, where plenty of empirical works have been done while the theoretical understanding is still largely missing \citep{pan2021plan,jiang2021offline,meng2021offline}. 
Offline RL has received tremendous attention because in various practical scenarios, it is expensive to acquire online data while offline log data is accessible. 

The offline single-agent RL is well studied in the literature. 
Researchers have identified the minimal dataset coverage assumption, \emph{single policy coverage} (the dataset only needs to cover an optimal policy), under which one can learn a near-optimal policy efficiently.
Furthermore, they have developed  algorithms with minimax sample complexity \citep{xie2021policy,li2022settling}.
For offline MARL, recent works showed that single policy coverage is not sufficient and \emph{unilateral coverage} is necessary for learning a Nash equilibrium (NE) strategy, i.e., the dataset covers all the joint strategies that only differ from an NE at one player \citep{cui2022offline,zhong2022pessimistic}. This condition is also sufficient for two-player zero-sum Markov games with sample complexity $\widetilde{O}(AB)$ (ignoring other quantities), where $A$, $B$ are the number of actions for each player \citep{cui2022offline}. However, it is still unclear if it is sufficient for multi-player general-sum Markov game.

One major challenge in MARL is the \emph{curse of multiagents}~\citep{jin2021v}. Suppose the number of actions for player $j$ is $A_j$ and there are $m$ players. Then the joint action space is of size $\prod_{j\in[m]}A_j$, which grows exponentially with the number of players $m$. As a result, any algorithm that depends linearly on the cardinality of the joint action space can hardly be applied to real-world scenarios. In online MARL, \citet{jin2021v} and \citet{song2021can} show that finding the coarse correlated equilibrium, which is a weaker equilibrium notion than NE,  only requires $\widetilde{O}(\max_{j\in[m]}A_j)$ samples, thus breaking the curse of multiagents. In this paper, we study the following question:

\begin{center}
    \emph{Can we find NE in offline $m$-player general-sum Markov game with unilateral coverage and without the exponential dependence on the number of players?}
\end{center}

In this paper, we answer this question in the affirmative.
We highlight our contributions below.

\subsection{Main Novelties and Contributions}


\textbf{1. Strategy-wise concentration principle.}
We propose the strategy-wise concentration principle. Point-wise concentration is a standard technique in computing the confidence interval for each state-action pair \citep{azar2017minimax,liu2021sharp,xie2021policy,cui2022offline}. However, the straightforward extension to MARL suffers from the curse of multiagents as the NE can be a mixed strategy. 
Different from the point-wise concentration technique, strategy-wise concentration directly \emph{estimates each strategy, which allows a tighter confidence interval that can avoid the dependence on the joint action space.} 
We give a technical overview in Section~\ref{sec:strategywise}.
In addition, we show that the strategy-wise confidence bound is always a convex function so that the empirical \emph{best response strategy can always be a deterministic strategy}, which is critical to the computational efficiency. 

\noindent 
\textbf{2. Improved algorithm for offline two-player zero-sum Markov games.}
For offline two-player zero-sum Markov games, we utilize its special structure to develop a maximin-optimization-type algorithm. Though the nonlinear strategy-wise bonus breaks the bilinear structure of the zero-sum game, we show that by solving a maximin optimization problem we can still output a good strategy. In addition, we can solve it efficiently using any black-box algorithms for Lipschitz-continuous convex optimization. Our sample complexity improves the $AB$ factor in \citet{cui2022offline} to $(A+B)$.

\noindent
\textbf{3. The first algorithm for offline multi-player general-sum Markov games.}
For multi-player general-sum Markov games, we develop a \emph{surrogate function} to approximate performance gap and then show that the minimizer of the surrogate function approximates NE well. The surrogate function is constructed by optimistic best response values and pessimistic values. Interestingly, to our knowledge, this is the first time that optimism has been used in offline RL algorithms. Our result validates that unilateral coverage is sufficient for general-sum Markov games and our sample complexity rate scales with $\widetilde{O}(\sum_{j=1}^mA_j)$ (ignoring other parameters), thus breaking the curse of multiagents.

\noindent
\textbf{4. Incorporating pre-specified strategy class.}
Lastly, our algorithm allows exploiting the prior knowledge about the NE strategy with an adaptive sample complexity bound.
Pre-specified policy class has been widely used in empirical works where the policy class is parameterized by neural networks (e.g., \cite{mnih2016asynchronous,haarnoja2018soft,lowe2017multi}), and single-agent RL theory as well (e.g., \cite{auer2002nonstochastic,agarwal2021theory}), but has not been investigated in MARL theory. In this paper, we take a step to incorporate prior knowledge in the MARL setting.
Our performance guarantee only depends on the logarithmic covering number of the pre-specified strategy class, which is always upper bounded by $\sum_{j\in[m]}A_j$, but can be smaller.
To the best of our knowledge, this is the first paper that considers a pre-specified strategy class in MARL theory.

\subsection{Technical Overview of Strategy-wise Concentration}
\label{sec:strategywise}
To give some intuition about this technique, let us consider a toy problem. Suppose there are $m$ random variables $\{x^i\}_{i=1}^m$ and we want to obtain a pessimistic estimate of their average $x=\sum_{i\in[m]}x^i/m$. We have $n/m$ observations for each $x^i$. 
The point-wise concentration estimate corresponds to estimating each $x^i$ and then aggregating the results.
The pessimistic estimate of $x^i$ would be $\widehat{x}^i-\widetilde{O}(\sqrt{m/n})$ where $\widehat{x}^i$ is the empirical mean, and the aggregated mean of these pessimistic estimates would be $\widehat{x}-\widetilde{O}(\sqrt{m/n})$ where $\widehat{x}$ is the empirical mean of all data. 
The strategy-wise concentration estimate corresponds to directly using all the samples to estimate the average of  $\{x\}_{i=1}^m$ and obtain the pessimistic estimate as $\widehat{x}-\widetilde{O}(1/\sqrt{n})$. 
This example shows that the point-wise estimate will lead to an extra $m$ factor. 
In MARL, $m$ is the cardinality of the joint action space, which implies that point-wise concentration can be exponentially worse than strategy-wise concentration. Note that this is not an issue in single-agent MDP as the optimal policy is always deterministic but leads to severe suboptimality in the multi-agent case where NE can be a mixed strategy.


\subsection{Related Work}

\textbf{Online Multi-agent RL.} Markov games can be solved via dynamic programming when the rewards and transition dynamics are given \citep{hansen2013strategy,perolat2015approximate}. If the environment is unknown, reinforcement learning algorithms are applied with different sampling oracles. One particular line of research is online Markov games, including two-player zero-sum Markov games \citep{liu2021sharp,dou2021gap,xie2020learning,bai2020near,huang2021towards} and multi-player general-sum Markov games \citep{zhong2021can,mao2021decentralized,jin2021v,song2021can}. \citet{rubinstein2016settling} proves an exponential (in the number of players) lower bound for learning the NE strategy in $m$-player general-sum game while others show that the correlated equilibrium and coarse correlated equilibrium admit $\mathrm{poly}(m,\max_{j\in[m]}A_j,H,S)$-sample complexity algorithms \citep{mao2021decentralized,jin2021v,song2021can}. Our upper bounds for $m$-player general-sum games depend polynomially on all parameters, which do not contradict the hardness result in \cite{rubinstein2016settling} because the assumptions on the offline dataset provide additional information about the NE.

\noindent
\textbf{Offline Single-agent RL.} The simplest dataset assumption for offline RL is uniform coverage, i.e., the dataset covers all the state-action pairs. This assumption dates back to \citet{szepesvari2005finite}. The minimax sample complexity has been well studied for both tabular case and function approximation \citep{xie2021batch,yin2020near,yin2021near,ren2021nearly}. Recently it has been shown that only covering the optimal policy is sufficient for offline RL under different settings \citep{rashidinejad2021bridging,yin2021towards,xie2021policy,jin2021pessimism,uehara2021pessimistic,zanette2021provable,xie2021bellman}. These works design provably efficient algorithms based on the principle of pessimism.

\noindent
\textbf{Offline Multi-agent RL.} Offline MARL theory is still at a primary stage. Previous works mostly focused on uniform coverage assumption, i.e. all state-action pairs or all policies are covered \citep{sidford2020solving,cui2021minimax,zhang2020model,zhang2021finite,abe2020off,subramanian2021robustness}. Recently, \citet{cui2022offline} and \citet{zhong2022pessimistic} show that the unilateral coverage assumption is the minimal dataset coverage assumption for learning NE in Markov games. In addition, \citep{cui2022offline} proposes a pessimism-type algorithm with $\widetilde{O}(SABH^3C(\pi^*)/\epsilon^2)$ sample complexity for tabular two-player zero-sum Markov game and \citep{zhong2022pessimistic} provides a similar algorithm for linear two-player zero-sum Markov games. 

\section{Preliminaries}

\textbf{Notations.}
We use $D(\mathcal{X})$ to denote the single point distributions over the finite set $\mathcal{X}$. For example, $D(\mathcal{A})$ to represent the policies that deterministically choose one of the actions in $\mathcal{A}$.
We use $\pi_{j,h}^s\in\Delta(\A_j)$ as a concise notation of $\pi_{j,h}(\cdot|s)$ and $P_h(s,\a)$ to denote $P_h(\cdot|s,\a)$, which will be defined in the following section. We use $-j$ in subscript to denote all the players except player $j$. We use bold letter to denote vectors, e.g. $\a$ is a vector and $a_j$ is the $j$-th element of $\a$. We let $O(\cdot)$ hide absolute constants and $\widetilde{O}(\cdot)$ hide $\mathrm{polylog}$ terms as well. The L1 norm of a vector in $\R^d$ is $\|\a\|_1=\sum_{i=1}^d|a_i|$. We denote the projection as $\proj_{[a,b]}(x):=\max\{a,\min\{b,x\}\}$.

\noindent
\textbf{Multi-player General-sum Markov Game.}
A multi-player general-sum Markov game is described by a tuple $\G=(\S,\A=\prod_{j\in[m]}\A_j,P,R,H)$, where $\S$ is the state space with cardinality $S$, $m$ is the number of players, $\A_j$ is the action space of player $j$ with cardinality $A_j$, $P=(P_1,P_2,\cdots,P_H)$ with $P_h\in\R^{S\times \prod_{i\in[m]}A_i\times S}$  being the (unknown) transition matrix at timestep $h \in [H]$, $R=\{R_{h}(\cdot|s_h,\a_h)\}_{h=1}^{H}$ with $R_{h}(\cdot|s_h,\a_h)$ being a distribution on $[0,1]^m$ with mean $\r_h(s_h,\a_h)\in[0,1]^m$ as the (unknown) reward distribution at timestep $h$. At timestep $h$, all players choose their actions \emph{simultaneously} and a reward vector is sampled from the reward distribution $\r_h\sim R_{h}(\cdot|s_h,\a_h)$, where $s_h$ is the current state and $\a_h=(a_{h,1},a_{h,2},\cdots,a_{h,m})$ is the joint action. Each player $j$ receives its own reward $r_{h,j}$  with support on $[0,1]$ and mean $r_{h,j}(s_h,\a_h)$. The state then transits to $s_{h+1}$ following the distribution of $P_h(\cdot\mid s_h,\a_h)$. The game terminates at timestep $H+1$. 
We assume that the initial state $s_1$ is fixed because  for a stochastic initial state, one can add $s_0$ as the initial state instead and it transits to $s_1$ following the initial distribution.

We denote a joint strategy as $\pi=(\pi_1,\pi_2,\cdots,\pi_m)$, where $\pi_j=(\pi_{1,j},\pi_{2,j},\cdots,\pi_{H,j})$ and $\pi_{h,j}:\S\rightarrow\Delta(\A_j)$ is the strategy of player $j$ at timestep $h$ where $\Delta(\A_j)$ is the probability simplex over $\A_j$. 
We use $\Pi^\mathrm{full}$ to denote the set of all the possible joint strategies. 
We define the state value function and state-action value function under strategy $\pi$ for each player $j\in[m]$:
$$V_{h,j}^\pi(s_h):=\E_\pi\Mp{\sum_{t=h}^H r_{t,j}(s_t,\a_t)\;\middle|\;s_h},Q_{h,j}^\pi(s_h,\a_h):=\E_\pi\Mp{\sum_{t=h}^H r_{t,j}(s_t,\a_t)\;\middle|\;s_h,\a_h},$$
where the expectation is over the randomness of the environment and the joint strategy $\pi$.
For a fixed player $j$, if all the other player's strategies are fixed, then player $j$ can play the best response strategy to maximize its own total reward. We define $\pi_{-j}$ to be the strategy for all players except player $j$ and  define the best response value to be
$$V_{h,j}^{*,\pi_{-j}}(s_h):=\max_{\pi_j}V_{h,j}^{\pi_j,\pi_{-j}}(s_h).$$

It is well-known that Nash equilibrium strategy exists for general-sum Markov games.
Note that there could be multiple NE strategies with different value functions. We use the following performance gap to evaluate a strategy $\pi$:
$$\mathrm{Gap}(\pi):=\sum_{j\in[m]}\Mp{V_{1,j}^{*,\pi_{-j}}(s_1)-V_{1,j}^\pi(s_1)}.$$
This metric is always non-negative and we say $\pi$ is an $\epsilon$-approximate NE if and only if $\mathrm{Gap}(\pi)\leq\epsilon$.

\noindent
\textbf{Two-player Zero-sum Markov Game.}
A general-sum Markov game becomes a two-player zero-sum Markov game if there are only two players and the reward $r_h\sim R_{h}(\cdot|s,a_1,a_2)$ always satisfies $r_{h,1}+r_{h,2}=0$ for all $h\in[H]$, $s\in\S$, $a_1\in\A_1$ and $a_2\in\A_2$. Following the literatures on two-player zero-sum Markov games, we use slightly different notations for this setting. There is only one reward function $r$ shared by both players, which is the reward function $\{r_{h,1}\}_{h=1}^H$ for player $1$ and the target of player $2$ is to minimize the total reward. We denote $\mu=\pi_1$ and $\nu=\pi_2$ to be the strategy for each player, $a=a_1$ and $b=a_2$ to be the action for each player, $\Pi^{\max}=\Pi_1$ and $\Pi^{\min}=\Pi_2$ to be the strategy class for each player to remove extra subscripts.
One can derive the performance gap under the new notations for two-player zero-sum Markov games: 
$$\mathrm{Gap}(\pi):=V_1^{*,\nu}(s_1)-V_1^{\mu,*}(s_1).$$

\noindent
\textbf{Offline Markov Game.}
In offline RL, the dataset is collected beforehand and no further sampling is allowed. Here we consider offline multi-player general-sum Markov game. The framework for offline two-player zero-sum Markov game is similar with the slightly different notations as we mentioned. 

We assume that the algorithm has access to an offline dataset $\D=\{(s_h^k,\a_{h}^k,\r_{h}^k,s_{h+1}^k)\}_{h,k=1,1}^{H,n}$ that satisfies Assumption \ref{asp:dataset}. The assumption states that the dataset is independently generated from the underlying Markov game, which is used in \citep{jin2021pessimism,zhong2022pessimistic}. The target of offline Markov game is to find a strategy $\pi$ with as small performance gap as possible by utilizing the dataset $\D$. One closely related assumption is that the dataset is generated from some behavior strategy \citep{xie2021policy,cui2022offline}. Though this kind of dataset does not satisfy Assumption \ref{asp:dataset} directly due to the dependence within the trajectory, we can construct a compliant dataset by using the subsampling technique in \citet{li2022settling} while the number of samples is still of the same order.
\begin{assumption}\label{asp:dataset}
The dataset $\D$ is compliant with the multi-player general-sum markov game, i.e.,
$$\P_\D(s_{h+1}^k=s\mid s_h^k,\a_h^k)=P_h(s_{h+1}=s\mid s_h=s_h^k,\a_h=\a_h^k),$$
$$\P_\D(\r_{h}^k=\r|s_h^k,\a_h^k)=R_{h}(\r_{h}=\r|s_h=s_h^k,a_h=\a_h^k),\forall j\in[m],$$
for all $h\in[H]$ and $k\in[n]$. In addition, all tuples $(s_h^k,\a_{h}^k,\r_{h}^k,s_{h+1}^k)$ are independent. 
\end{assumption}

\noindent
\textbf{Pre-specified Policy Class.}
We also consider the case when we know that the NE is possibly in a given subset of $\Pi^{\mathrm{full}}$. 
We denote this subset as $\Pi$ and our target is to find the best strategy in $\Pi$. Note that we do not assume NE is indeed in $\Pi$. In addition, by choosing $\Pi=\Pi^\mathrm{full}$ we can recover the standard setting. To measure the complexity of $\Pi$, we use the covering number.

\begin{definition}\label{def:N}
(Covering Number)
For any error level $\epsilon_{\mathrm{cover}}$ and strategy class $\Pi$, we define
\[\N(\Pi,\epsilon_{\mathrm{cover}}):=\sum_{s\in\S,h\in[H]}\prod_{j\in[m]}\abs{\C(\Pi_{h,j}(s),\epsilon_\mathrm{cover})},\]
where $\Pi_{h,j}(s)=\{\pi_h^j(\cdot|s):\pi\in\Pi\}$ is a subset of $\Delta(\A_i)$ and $\C(\Pi_{h,j}(s),\epsilon_\mathrm{cover})$ is an $\epsilon_\mathrm{cover}$-covering of $\Pi_{h,j}(s)$ with respect to the L1 norm $\|\cdot\|_1$.
\end{definition}
Our performance guarantee will only have logarithm dependence on $\N(\Pi,\epsilon_{\mathrm{cover}})$. As $\Pi_{h,j}(s)$ is a subset of $\Delta(\A_j)$, we always have $\log(\N(\Pi,\epsilon_{\mathrm{cover}}))\leq \widetilde{O}(\sum_{j\in[m]}A_j\log(1/\epsilon_\mathrm{cover}))$ and if $\Pi$ is a finite set, we have $\log(\N(\Pi,\epsilon_{\mathrm{cover}}))\leq\log(SH|\Pi|)$ (see Appendix \ref{apx:covering} for the proof). In this paper we will choose $\epsilon_\mathrm{cover}=\frac{1}{\sum_{j\in[m]}A_j mH^2n^2}$, which only leads to logarithm dependence on these quantities. In later sections, we will omit $\epsilon_\mathrm{cover}$ to simplify the notation.

For any joint strategy $\pi$, we call $(\pi'_j,\pi_{-j})$ for any strategy $\pi'$ and $j\in[m]$ as a unilateral strategy of $\pi$. Previous works show that only covering an NE is not sufficient, and covering all the unilateral strategies of an NE is necessary for learning the NE in Markov games \citep{cui2022offline,zhong2022pessimistic}. We use unilateral coefficient to quantify how the dataset covers all the unilateral strategies of a strategy $\pi$. If we assume that the dataset is sampled from some (unknown) distribution, i.e. $(s_h,\a_h)\sim d_h(\cdot,\cdot)$ for all $h\in[H]$, we can define the population unilateral coefficient.
\begin{definition}\label{def:policy}
For any strategy $\pi$, the population unilateral coefficient is defined as
$$C(\pi):=\max_{h,j,\pi',s_h,\a_h}\frac{d_h^{\pi'_j,\pi_{-j}}(s_h,\a_h)}{d_h(s_h,\a_h)}.$$
\end{definition}
\vspace{-0.1cm}

\citet{cui2022offline} provide a sample complexity result for zero-sum Markov games with dependence on $C(\pi^*)$. We can also define the empirical unilateral coefficient using the empirical distribution. 
\begin{definition}\label{def:marl dataset}
Define the empirical dataset distribution as
$\widehat{d}_h(s,\a)=n_h(s,\a)/n$, for all $h\in[H],s\in\S,\a\in\A$, where $n_h(s,\a)$ is the number of times that $(s,\a)$ appears in the dataset for timestep $h$. For any strategy $\pi$, the empirical unilateral coefficient is defined as
$$\widehat{C}(\pi):=\max_{h,j,\pi',s_h,\a_h}\frac{d_h^{\pi'_j,\pi_{-j}}(s_h,\a_h)}{\widehat{d}_h(s_h,\a_h)}.$$
\end{definition}
\vspace{-0.1cm}

The empirical unilateral coefficient can lead to dataset-dependent bound that has no dependence on the underlying distribution of the dataset. In addition, $\widehat{C}(\pi)$ can be bounded by $2C(\pi)$ (Proposition \ref{prop:marl warmup}) so results based on $\widehat{C}(\pi)$ directly transfer to $C(\pi)$. Note that $\widehat{C}(\pi)$ and $C(\pi)$ are both unknown to the algorithm and only appear in the analysis and theorems.
\vspace{-0.1cm}
\begin{proposition}\label{prop:marl warmup}
Suppose $p_\mathrm{min}=\min_{s,\a,h}\{d_h(s,\a):d_h(s,\a)>0\}$. If $n\geq \frac{8\log(S\Pi_{j\in[m]}A_jH/\delta)}{p_\mathrm{min}}$, with probability $1-\delta$, for all strategy $\pi$, we have $2C(\pi)\geq \widehat{C}(\pi). $
\end{proposition}
\vspace{-0.2cm}

\section{An Improved Algorithm for Offline Two-player Zero-sum Markov Game}\label{sec:tpzs}

\setlength{\textfloatsep}{0.1cm}
\begin{algorithm}[!t]
    \caption{\textbf{S}trategy-wise \textbf{B}onus + \textbf{M}axi\textbf{M}in Optimization (\algonanametwo)}
	\label{algo:zerosum markov game}
    \begin{algorithmic}
        \STATE Input: offline dataset $\mathcal{D}$.
        \STATE Initialization: $\uv_{H+1}(s)=\ov_{H+1}(s)=0$ for all $s\in\S$. 
        \FOR{time $h=H, H-1, \dots, 1$}
        \STATE \#Player 1
            \STATE Approximately solve $\um_h^s= \argmax_{\mu_h^s\in \Pi_{h}^{\mathrm{max}}(s)}\min_{\nu_h^s\in D(\B)}\uv^{\mu_h^s,\nu_h^s}_h(s)$, where $\uv^{\mu_h^s,\nu_h^s}_h(s)$ is defined by \eqref{eq:uq} and \eqref{eq:uv} and $\um_h^s$ satisfies \eqref{eq:maxmin epsilon}.
            \STATE Solve $\un_h^s=\argmin_{\nu_h^s\in D(\B)}\uv^{\um_h^s,\nu_h^s}_h(s)$ and set $\uv_h(s)=\proj_{[0,H-h+1]}\Bp{\uv_h^{\um_h^s,\un_h^s}(s)}$.
            \STATE \#Player 2
            \STATE Approximately solve $\on_h^s = \argmin_{\nu_h^s\in \Pi_{h}^{\mathrm{min}}(s)}\max_{\mu_h^s\in D(\A)}\ov^{\mu_h^s,\nu_h^s}_h(s)$, where $\ov^{\mu_h^s,\nu_h^s}_h(s)$ is defined by \eqref{eq:oq} and \eqref{eq:ov} and $\on_h^s$ satisfies \eqref{eq:minmax epsilon}.
            \STATE Solve $\om_h^s=\argmax_{\mu_h^s\in D(\A)}\ov^{\mu_h^s,\on_h^s}_h(s)$ and set $\ov_h^s=\proj_{[0,H-h+1]}\Bp{\ov_h^{\om_h^s,\on_h^s}(s)}$.
        \ENDFOR
	    \STATE Output $\pi^\mathrm{output}=(\um,\on)$.
    \end{algorithmic}
\end{algorithm}

In this section, we propose a new algorithm for offline zero-sum Markov game based on two novel techniques, i.e., strategy-wise concentration and maximin-optimization-based algorithm. We then show that this algorithm is computationally efficient and can (almost) find the best strategy in strategy class $\Pi$ with favorable sample complexity.

Let us first define some notations. Given a dataset $\mathcal{D}=\{(s_h^k,a_h^k,b_h^k,r_h^k,s_{h+1}^k)\}_{k,h=1}^{n,H}$, we denote $n_h(s,a,b)=\sum_{k=1}^n\1\Sp{(s_h^k,a_h^k,b_h^k)=(s,a,b)}$ and $\K_h(s)=\{(a,b)\in\A\times\B:n_h(s,a,b)\neq0\}$. If $n_h(s,a,b)\neq0$, we set

\begin{equation*}
    \widehat{r}_{h}(s,a,b)=\frac{\sum_{k=1}^n r_h^k\1\Sp{(s_h^k,a_h^k,b_h^k)=(s,a,b)}}{n_h(s,a,b)},
\end{equation*}
\begin{equation*}
    \widehat{P}_h(s'|s,a,b)=\frac{\sum_{k=1}^n\1\Sp{(s_h^k,a_h^k,b_h^k,s_{h+1}^k)=(s,a,b,s')}}{n_h(s,a,b)},
\end{equation*}
otherwise we have
\begin{equation}\label{eq:r n=0}
 \widehat{r}_h(s,a,b)=0,\widehat{P}_h(s'|s,a,b)=0.
\end{equation}
Based on this empirical Markov game, we can perform value-iteration-type algorithm. Here we describe our algorithm for player 1.  For each timestep $h$, we first compute the the state-action values based on the estimates at timestep $h+1$:
\begin{equation}\label{eq:uq}
    \uq_h(s,a,b)=\widehat{r}_h(s,a,b)+\inner{\widehat{P}_h(s,a,b),\uv_{h+1}},
\end{equation}
Then instead of adding the bonus on state-action estimates directly to ensure pessimism as used in \citet{cui2022offline} and \citet{zhong2022pessimistic}, we first estimate the state value functions for strategy $\mu_h^s,\nu_h^s$ and then add the bonus on them instead. 
\begin{equation}\label{eq:uv}
    \uv_h^{\mu_h^s,\nu_h^s}(s)=\E_{a\sim\mu_h^s,b\sim\nu_h^s}\uq_h(s,a,b)-b_h(s,\mu_h^s,\nu_h^s),
\end{equation}
where
\begin{equation}\label{eq:zsmg bonus}
 b_h(s,\mu_h^s,\nu_h^s)=H\sqrt{\sum_{(a,b)\in\K_h(s)}\frac{\mu^s_h(a)^2\nu^s_h(b)^2}{n_h(s,a,b)}\log(\N(\Pi))\iota}+\sqrt{\iota}/n,
\end{equation}
with $\iota=32\log(2ABSHn/\delta)$. We also present the bonus from point-wise concentration used in \cite{cui2022offline} to better compare them,
$$
    b^{\mathrm{point}}_h(s,\mu_h^s,\nu_h^s)=H\sum_{(a,b)\in\K_h(s)}\mu^s_h(a)\nu^s_h(b)\sqrt{\frac{\iota}{n_h(s,a,b)}}.
$$

As a concrete example, if $\mu_h^s$ and $\nu_h^s$ are uniform distribution on $\A$ and $\B$, then $b_h(s,\mu_h^s,\nu_h^s)$ is smaller than $b^{\mathrm{point}}_h(s,\mu_h^s,\nu_h^s)$ for an order of $\sqrt{AB}$. 
Finally to obtain the pessimistic value estimate, we solve the following optimization problem
\begin{equation}\label{eq:maxmin}
    \uv_h(s)=\max_{\mu_h^s\in \Pi^{\mathrm{max}}_{h}(s)}\min_{\nu_h^s\in D(\B)}\uv^{\mu_h^s,\nu_h^s}_h(s).
\end{equation}
Here recall that $D(\B)$ represents all the deterministic strategies in $\B$.
Our algorthm is similar for player 2 with the following $\oq$ and $\ov$ estimation:
\begin{equation}\label{eq:oq}
    \oq_h(s,a,b)=\widehat{r}_h(s,a,b)+\inner{\widehat{P}_h(s,a,b),\ov_{h+1}}+H\1\{(a,b)\notin\K_h(s)\},
\end{equation}
\begin{equation}\label{eq:ov}
    \ov_h^{\mu_h^s,\nu_h^s}(s)=\E_{\mu_h^s,\nu_h^s}\oq_h(s,a,b)+b_h(s,\mu_h^s,\nu_h^s). 
\end{equation}
The additional $H\1\{(a,b)\notin\K_h(s)\}$ term in \eqref{eq:oq} compared with  \eqref{eq:uq} is to compensate the underestimate by \eqref{eq:r n=0}. 
\subsection{Computational Efficiency}
For computational efficiency, we start with the following characterization about our bonus.
 \begin{proposition}\label{prop:concave}
$\uv_h^{\mu_h^s,\nu_h^s}(s)$ is concave and $\ov_h^{\mu_h^s,\nu_h^s}(s)$ is convex  w.r.t. $\mu_h^s$ and $\nu_h^s$ respectively. 
\end{proposition}

Proposition \ref{prop:concave} explains why the inner minimization in \eqref{eq:maxmin} is over the deterministic strategy class as the minimum of a concave function over the probability simplex is achieved at the vertexes, i.e. deterministic strategies. The proof of Proposition \ref{prop:concave} is provided in Appendix \ref{apx:convexity}. 

Previous works solve the NE (saddle point) of $\uv^{\mu_h^s,\nu_h^s}_h(s)$ as the point-wise bonus maintains the bilinear structure \citep{cui2022offline,zhong2022pessimistic}. Though here $\uv^{\mu_h^s,\nu_h^s}_h(s)$ no longer enjoys the strong duality, we will show that solving the maximin problem is enough to obtain a good strategy for player 1. As the inner minimization is only on a feasible set of size $B$, this problem can be solved efficiently by using projected gradient descent \citep{bubeck2015convex}.
We assume that we solve the maximin and the minimax optimization problem to $\epsilon_\mathrm{opt}$-optimality, i.e.,
\begin{equation}\label{eq:maxmin epsilon}
    \min_{\nu_h^s\in D(\B)}\uv^{\um_h^s,\nu_h^s}_h(s)\geq\max_{\mu_h^s\in \Pi_{h}^{\mathrm{max}}(s)}\min_{\nu_h^s\in D(\B)}\uv^{\mu_h^s,\nu_h^s}_h(s)-\epsilon_\mathrm{opt},
\end{equation}
\begin{equation}\label{eq:minmax epsilon}
    \max_{\mu_h^s\in D(\A)}\ov^{\mu_h^s,\on_h^s}_h(s)\leq\min_{\nu_h^s\in \Pi_{h}^{\mathrm{min}}(s)}\max_{\mu_h^s\in D(\A)}\ov^{\mu_h^s,\nu_h^s}_h(s)+\epsilon_\mathrm{opt}.
\end{equation}
 In Appendix \ref{apx:convexity} we show that projected gradient descent can output an $\epsilon_\mathrm{opt}$-minimizer with $(H+H\sqrt{\log(\N(\Pi))\iota})/\epsilon_{\mathrm{opt}}^2$ iterations, where each iteration consists of a gradient computation and a projection onto the probability simplex. 
 We note that if we set $\epsilon_{\mathrm{opt}}$ to $\frac{1}{\sqrt{n}}$, then the optimization error is always of a smaller order term compared to the statistical error.

\subsection{Sample Complexity Guarantees for \algonanametwo}
 For the statistical guarantee, we will first provide \emph{assumption-free bounds} in the sense that it holds for arbitrary compliant dataset \citep{jin2021pessimism,yin2021towards}. We define the uncertainty at timestep $h$ and state $s$ under strategy $\mu_h^s$ and $\nu_h^s$: $$\widehat{b}_h(s,\mu_h^s,\nu_h^s):=2b_h(s,\mu_h^s,\nu_h^s)+H\sum_{(a,b)\notin\K_h(s)}\mu_h^s(a)\nu_h^s(b).$$ 

\begin{proposition}\label{prop:tpzs asp free}
Suppose $\pi^{\mathrm{output}}$ is the output of Algorithm \ref{algo:zerosum markov game}. With probability $1-\delta$, we have

\begin{align*}
    &\mathrm{Gap}(\pi^{\mathrm{output}}) \leq\\
    &  \min_{\pi=(\mu,\nu)\in\Pi}\max_{\pi'=(\mu',\nu')\in \Pi^{\mathrm{det}}}\left[\mathrm{Gap}(\pi)+\E_{\mu,\nu'}\sum_{h=1}^H\widehat{b}_h(s_h,\mu^{s_h}_h,\nu'^{s_h}_h)\right.\left.+\E_{\mu',\nu}\sum_{h=1}^H\widehat{b}_h(s_h,\mu'^{s_h}_h,\nu_h)\right]+2H\epsilon_{\mathrm{opt}}.
\end{align*}
\end{proposition}
Proposition \ref{prop:tpzs asp free} shows that our algorithm can find the best strategy in $\Pi$ with an additional error of the expected total uncertainty under some unilateral strategies and an extra optimization error term $2H\epsilon_\mathrm{opt}$.  Then we derive bounds with unilateral coefficients. 
\begin{theorem}\label{thm:dataset main}
Suppose $\pi^{\mathrm{output}}$ is the output of Algorithm \ref{algo:zerosum markov game}. With probability $1-\delta$, we have
$$\mathrm{Gap}(\pi^{\mathrm{output}})\leq\min_{\pi\in\Pi}\Mp{\mathrm{Gap}(\pi)+4H^2\sqrt{S\log(\N(\Pi))\widehat{C}(\pi)\iota/n}}+2H\epsilon_{\mathrm{opt}}.$$
\end{theorem}
Theorem \ref{thm:dataset main} directly implies the following corollary. 
\begin{corollary}\label{crl:full}
If $\Pi=\Pi^\mathrm{full}$, then with probability $1-\delta$, we have
$$\mathrm{Gap}(\pi^{\mathrm{output}})=\widetilde{O}(\sqrt{H^4S(A+B)\widehat{C}(\pi^*)/n})+2H\epsilon_{\mathrm{opt}}.$$
If $\pi^*\in\Pi$, then with probability $1-\delta$, we have
$$\mathrm{Gap}(\pi^{\mathrm{output}})=\widetilde{O}(\sqrt{H^4S\log(\N(\Pi))\widehat{C}(\pi^*)/n})+2H\epsilon_{\mathrm{opt}}.$$
\end{corollary}

Since $\widehat{C}(\pi)$ can be bounded using $C(\pi)$ (Proposition \ref{prop:marl warmup}), we have the following theorem. 

\begin{theorem}\label{thm:policy main}
Suppose $\pi^{\mathrm{output}}$ is the output of Algorithm \ref{algo:zerosum markov game}. With probability $1-\delta$, we have
$$\mathrm{Gap}(\pi^{\mathrm{output}})\leq\min_{\pi\in\Pi}\Mp{\mathrm{Gap}(\pi)+4H^2\sqrt{S\log(\N(\Pi))C(\pi)\iota^2/n}+HS(A+B)C(\pi)/n}+2H\epsilon_{\mathrm{opt}}.$$
In addition, suppose $p_\mathrm{min} =\min_{s,a,b,h}\{d_h^\rho(s,a,b):d_h^\rho(s,a,b)>0\}$ and if $n\geq \frac{8\log(SABH/\delta)}{p_\mathrm{min}}$, we have
$$\mathrm{Gap}(\pi)\leq\min_{\pi\in\Pi}\Mp{\mathrm{Gap}(\pi)+8H^2\sqrt{S\log(\N(\Pi))C(\pi)\iota^2/n}}+2H\epsilon_{\mathrm{opt}}.$$
\end{theorem}
 Theorem \ref{thm:policy main} shows that there will be an additional lower order term $S(A+B)C(\pi)/n$, which can be interpreted as the rate of the empirical dataset distribution converges to the population distribution. In addition, for large enough $n\geq \frac{8\log(SABH/\delta)}{p_\mathrm{min}}$, there is no lower order term. Here $n\geq \frac{8\log(SABH/\delta)}{p_\mathrm{min}}$ serves as a warm-up cost so that the empirical support is the same as the true support of $d_h$. A similar analysis is used in \citet{yin2021towards}.  With a refined analysis, we can show that there is no lower order term for the standard settings $\Pi=\Pi^\mathrm{full}$ in two-player zero-sum Markov games and $\Pi=\Pi^\mathrm{det}$ for turn-based Markov games. Note that turn-based Markov games always have a deterministic NE.

\begin{corollary}\label{crl:full main}
If $\Pi=\Pi^{\mathrm{full}}$, then with probability $1-\delta$, we have
$$\mathrm{Gap}(\pi^{\mathrm{output}})=\widetilde{O}(\sqrt{H^4S(A+B)C(\pi^*)/n})+2H\epsilon_{\mathrm{opt}}. $$
In addition, for turn-based two-player zero-sum Markov games, we can set $\Pi=\Pi^{\mathrm{det}}$ and we have
$$\mathrm{Gap}(\pi^{\mathrm{output}})=\widetilde{O}(\sqrt{H^4SC(\pi^*)/n})+2H\epsilon_{\mathrm{opt}}.$$
\end{corollary}
Corollary \ref{crl:full main} improves the $AB$ dependence in the previous zero-sum Markov games result \citep{cui2022offline} and matches the result for turn-based Markov games \citep{cui2022offline} up to an extra $\sqrt{H}$ factor. The additional $H$ factor is due to the Hoeffding-type bonus and we believe it can be removed with a more sophisticated Bernstein-type bonus. 

\section{Algorithms and Analyses for Multi-player General-sum Markov Game}\label{sec:mpgs}

In this section, we propose the first provably efficient algorithm for offline multi-player general-sum Markov game. We will use the strategy-wise bonus to achieve a sample complexity that does not scale with $\prod_{j\in[m]}A_j$. However, in general-sum games there is no saddle point structure, so we can no longer use the maximin-optimization-type algorithm. Instead, our algorithm utilizes a novel \emph{surrogate function} to approximately minimize the performance gap.

Given a dataset $\D=\{(s_h^k,\a_{h}^k,\r_{h}^k,s_{h+1}^k)\}_{k,h=1}^{n,H}$, we denote $n_h(s,\a)=\sum_{k=1}^n\1\Sp{(s_h^k,\a_h^k)=(s,\a)}$ and $\K_h(s)=\{\a:n_h(s,\a)\neq0\}$. If $n_h(s,\a)>0$, we set
\begin{equation}
    \widehat{r}_{h,j}(s,\a)=\frac{\sum_{k=1}^n r_{h,j}^k\1\Sp{(s_h^k,\a_h^k)=(s,\a)}}{n_h(s,\a)},\widehat{P}_h(s'|s,\a)=\frac{\sum_{k=1}^n\1\Sp{(s_h^k,\a_h^k,s_{h+1}^k)=(s,\a,s')}}{n_h(s,\a)},
\end{equation}
otherwise we have $\widehat{r}_{h,j}(s,\a)=0,\widehat{P}_h(s'|s,\a)=0.$

Based on this empirical multi-player Markov game, we can estimate the value of arbitrary strategy $\pi$ via policy evaluation (Algorithm \ref{alg:value estimation} in Appendix). We describe Algorithm \ref{alg:value estimation} for the pessimistic estimate. For a player $j$, strategy $\pi$ and timestep $h$, we first compute the state-action value estimates:
\begin{equation}\label{eq:marl uq}
    \uq^\pi_{h,j}(s,\a)=\widehat{r}_{h,j}(s,\a)+\inner{\widehat{P}_h(s,\a),\uv^\pi_{h+1,j}},
\end{equation}
Then we estimate the state value functions and add the strategy-wise bonus to ensure pessimism.
\begin{align}
    \uv_{h,j}^{\pi}(s)=&\proj_{[0,H-h+1]}\Bp{\E_{\a\sim\pi_h(\cdot|s)}\uq^\pi_{h,j}(s,\a)-b_h(s,\pi_h^s)}, \label{eq:marl uv}
\end{align}
where
\begin{align}
 b_h(s,\pi^s_h)=&H\sqrt{\sum_{\a\in\K_h(s)}\frac{\prod_{j \in [m]}\pi_{h,j}^s(a_j)^2}{n_h(s,\a)}S\log(\N(\Pi))\iota}+\sqrt{\iota}/n, \label{eqn:marl bonus}
\end{align}
with $\iota=32\log(16\prod_{j\in[m]}A_jmSHn/\delta)$. Here the strategy-wise pessimism can remove the $\prod_{j\in[m]}A_j$ dependence as explained in the previous section. By dynamic programming from timestep $H$ to timestep $1$ we can obtain the pessimistic estimate $\uv_{1,j}^\pi(s_1)$. Compared with the bonus function \eqref{eq:zsmg bonus} in zero-sum Markov game, there is an extra $S$ factor in \eqref{eqn:marl bonus} because here we need to perform concentration on $\inner{\widehat{P}_h(s,\a),\uv^\pi_{h+1,j}}$ for all $\pi$ while in \eqref{eq:uq} we only need to analyze $\inner{\widehat{P}_h(s,a,b),\uv_{h+1}}$ for a single $\uv_{h+1}$. We use an additional $\epsilon$-covering on $\R^S$ which leads to the extra $S$.

We use Algorithm \ref{alg:best response estimation} (in Appendix) to compute the optimistic value of the best response strategy. For a given player $j$, strategy $\pi_{-j}$ used by all the other player and timestep $h$, we first compute the optimistic state-action value estimate:
\begin{equation}\label{eq:marl oq}
    \oq^{*,\pi_{-j}}_{h,j}(s,\a)=\widehat{r}_{h,j}(s,\a)+\inner{\widehat{P}_h(s,\a),\ov^{*,\pi_{-j}}_{h+1,j}}+H\1\{\a\notin\K_h(s)\}.
\end{equation}
Then we compute the optimistic value for deterministic strategies for player $j$:
\begin{align}
    \ov_{h,j}(s,a_j)=&\E_{\a_{-j}\sim\pi_{h,-j}(\cdot|s)}\oq^{*,\pi_{-j}}_{h,j}(s,a_j,\a_{-j})+b_h(s,a_j,\pi_{h,-j}^s).
\end{align}
Here with a slight abuse of the notation, we use $a_j$ to denote the deterministic strategy of player $j$ that chooses action $a_j$ at state $s$ and timestep $h$. Finally we use the maximum over all the deterministic strategies to be the best response value function:$$
    \ov_{h,j}^{*,\pi_{-j}}(s)=\proj_{[0,H-h+1]}\Bp{\max_{a_j\in\A_j}\ov_{h,j}(s,a_j)}.
$$

By dynamic programming we can obtain the optimistic estimate $\ov_{1,j}^{*,\pi_{-j}}(s_1)$ at the initial state. Note that we only consider the deterministic strategies for player $j$. Thanks to the convexity of the bonus $b_h(s,\pi_h^s)$, the best response with respect to $\ov_{h,j}^\pi(s)$ is also in the deterministic strategy class as in zero-sum Markov games. The following proposition connects Algorithm \ref{alg:value estimation} and Algorithm \ref{alg:best response estimation}:
\begin{proposition}
For any strategy $\pi_{-j}\in\Pi^{\mathrm{full}}_{-j}$,$h\in[H]$ and $s\in\S$, we have
$\ov_{h,j}^{*,\pi_{-j}}(s)=\max_{\pi_j}\ov_{h,j}^{\pi_j,\pi_{-j}}(s)$.
\end{proposition}

Based on Algorithm \ref{alg:value estimation} and Algorithm \ref{alg:best response estimation}, we propose a surrogate minimization algorithm for multi-player general-sum Markov game. Suppose $\uv_{1,j}^\pi(s_1)$ and $\ov_{1,j}^{*,\pi_{-j}}(s_1)$ are pessimistic and optimistic estimates, then we have
$$\mathrm{Gap}(\pi)=\sum_{j\in[m]}V_{1,j}^{*,\pi_{-j}}(s_1)-V_{1,j}^\pi(s_1)\leq \sum_{j\in[m]}\ov_{1,j}^{*,\pi_{-j}}(s_1)-\uv_{1,j}^\pi(s_1).$$
The RHS can serve as the surrogate function and \algonamemulti~ (Algorithm \ref{algo:generalsum markov game} in Appendix) outputs the minimizer of it in $\Pi$. From the computational perspective, Algorithm \ref{alg:value estimation} and Algorithm \ref{alg:best response estimation} are both efficient while Algorithm \ref{algo:generalsum markov game} needs to enumerate $\Pi$ for the worst case. This computational hardness agrees with the PPAD-hardness for computing approximate NE even in full information general-sum game \citep{daskalakis2013complexity}. However, if $\Pi$ is well structured, Algorithm \ref{algo:generalsum markov game} may be computationally efficient and we leave it to future work. Here we assume $\pi^{\mathrm{output}}$ is an exact solution while it is straightforward to incorporate optimization error as in the previous section. 

\subsection{Sample Complexity Guarantees for \algonamemulti}
We still begin with assumption-free bound as in the previous section. We define the uncertainty at timestep $h$ and state $s$ under strategy $\pi$: $$\widehat{b}_h(s,\pi^s_h)=2b_h(s,\pi^s_h)+H\sum_{\a\notin\K_h(s)}\pi^s_h(\a).$$
\begin{proposition}\label{prop:marl asp free}
Suppose $\pi^\mathrm{output}$ is the output of Algorithm \ref{algo:generalsum markov game}. With probability $1-\delta$, we have
{\small
\begin{align*}
    \mathrm{Gap}(\pi^\mathrm{output})\leq \min_{\pi\in\Pi}\Mp{\mathrm{Gap}(\pi)+\max_{\pi'\in \Pi^{\mathrm{det}}}\sum_{j\in[m]}\E_{\pi'_j,\pi^*_{-j}}\sum_{h=1}^H \widehat{b}_h(s_h,\pi'^{s_h}_{h,j},\pi_{h,-j}^{s_h})+m\E_{\pi}\sum_{h=1}^H \widehat{b}_h(s_h,\pi_h^{s_h})} . 
\end{align*}
}%
\end{proposition}
\vspace{-0.2cm}

Proposition \ref{prop:marl asp free} has a similar structure as Proposition \ref{prop:tpzs asp free} with a slight difference in the expected uncertainty terms. Then we will bound using the unilateral coefficients.   

\begin{theorem}\label{thm:marl dataset main}
Suppose $\pi^\mathrm{output}$ is the output of Algorithm \ref{algo:generalsum markov game}. With probability $1-\delta$, we have
\[\mathrm{Gap}(\pi^{\mathrm{output}})\leq\min_{\pi\in\Pi}\Mp{\mathrm{Gap}(\pi)+4mH^2S\sqrt{\widehat{C}(\pi)\log(\N(\Pi))\iota/n}}.\]
\end{theorem}
\vspace{-0.2cm}
Theorem \ref{thm:marl dataset main} directly implies the following corollary, which shows that the sample complexity of offline multi-agent RL only scales linearly with respect to the number of the players.
\begin{corollary}\label{crl:full marl}
If $\Pi=\Pi^{\mathrm{full}}$, with probability $1-\delta$, we have
$$\mathrm{Gap}(\pi^{\mathrm{output}})=\widetilde{O}(\sqrt{H^4S^2\sum_{j\in[m]}A_j\widehat{C}(\pi^*)/n}).$$
If $\pi^*\in\Pi$, then with probability $1-\delta$, we have
$$\mathrm{Gap}(\pi^{\mathrm{output}})=\widetilde{O}(\sqrt{H^4S^2\log(\N(\Pi))\widehat{C}(\pi^*)/n}).$$
\end{corollary}
\vspace{-0.2cm}
Similarly we have the following theorem and corollary for the population unilateral coefficient. 
\begin{theorem}\label{thm:marl policy main}
Suppose $\pi^\mathrm{output}$ is the output of Algorithm \ref{algo:generalsum markov game}. If $n\geq \frac{8\log(S\Pi_{j\in[m]}A_jH/\delta)}{p_\mathrm{min}}$, with probability $1-\delta$, we have
$$\mathrm{Gap}(\pi^{\mathrm{output}})\leq\min_{\pi\in\Pi}\Mp{\mathrm{Gap}(\pi)+4mH^2S\sqrt{2C(\pi)\log(\N(\Pi))\iota/n}}.$$
\end{theorem}
\begin{corollary}\label{crl:full marl 2}
Suppose $n\geq \frac{8\log(S\Pi_{j\in[m]}A_jH/\delta)}{p_\mathrm{min}}$. If $\Pi=\Pi^{\mathrm{full}}$, with probability $1-\delta$, we have
$$\mathrm{Gap}(\pi^{\mathrm{output}})=\widetilde{O}\left(\sqrt{H^4S^2\sum_{j\in[m]}A_jC(\pi^*)/n}\right).$$
If $\pi^*\in\Pi$, then with probability $1-\delta$, we have
$$\mathrm{Gap}(\pi^{\mathrm{output}})=\widetilde{O}(\sqrt{H^4S^2\log(\N(\Pi))C(\pi^*)/n}).$$
\end{corollary}
\vspace{-0.2cm}


\section{Conclusion}
In this work, we studied offline MARL.
With a novel strategy-wise bonus, we remove the exponential dependence on the number of players. We use different algorithm frameworks for zero-sum Markov games and general-sum Markov games due to their different properties. 

Here we list several open problems for future work. One direction is to find the minimax sample complexity for offline Markov games, i.e., if the $\log(\N(\Pi))$ term is necessary. Another direction is to design computationally efficient algorithms for finding (coarse) correlated equilibrium in general-sum Markov games. Lastly, we only focus on the tabular setting serving as a start point. It is important to study MARL with reasonable function approximation.

\section*{Acknowledgements}
This work was supported in part by NSF CCF 2212261, NSF IIS 2143493, NSF DMS-2134106, NSF CCF 2019844 and NSF IIS 2110170.

\bibliography{reference}
\bibliographystyle{plainnat}


\newpage
\appendix

\section{Algorithms}

\begin{algorithm}[!htbp]
    \caption{Value Estimation}
	\label{alg:value estimation}
    \begin{algorithmic}
        \STATE Input: offline dataset $\mathcal{D}$, player index $j$ and strategy $\pi$.
        \STATE Initialization: $\uv^\pi_{H+1,j}(s)=\ov^\pi_{H+1,j}(s)=0$ for all $s\in\S$. 
        \FOR{time $h=H, H-1, \dots, 1$}
            \STATE Set $\uq^\pi_{h,j}(s,\a)=\widehat{r}_{h,j}(s,\a)+\inner{\widehat{P}_h(s,\a),\uv^\pi_{h+1,j}}$
            \STATE Set
            $\uv_{h,j}^{\pi}(s)=\proj_{[0,H-h+1]}\Bp{\E_{\a\sim\pi_h(\cdot|s)}\uq^\pi_{h,j}(s,\a)-b_h(s,\pi_h^s)}$
            \STATE Set $\oq^\pi_{h,j}(s,\a)=\widehat{r}_{h,j}(s,\a)+\inner{\widehat{P}_h(s,\a),\ov^\pi_{h+1,j}}+H\1\{\a\notin\K_h(s)\}$
            \STATE Set $\ov_{h,j}^{\pi}(s)=\proj_{[0,H-h+1]}\Bp{\E_{\a\sim\pi_h(\cdot|s)}\oq^\pi_{h,j}(s,\a)+b_h(s,\pi_h^s)}$
        \ENDFOR
	    \STATE Output $\uv_{1,j}^\pi(s_1)$ and $\ov_{1,j}^\pi(s_1)$.
    \end{algorithmic}
\end{algorithm}

\begin{algorithm}[!htbp]
    \caption{Best Response Estimation}
	\label{alg:best response estimation}
    \begin{algorithmic}
        \STATE Input: offline dataset $\mathcal{D}$, player index $j$ and strategy $\pi_{-j}$.
        \STATE Initialization: $\ov^{*,\pi_{-j}}_{H+1,j}(s)=0$ for all $s\in\S$. 
        \FOR{time $h=H, H-1, \dots, 1$}
        \STATE Set $\oq^{*,\pi_{-j}}_{h,j}(s,\a)=\widehat{r}_{h,j}(s,\a)+\inner{\widehat{P}_h(s,\a),\ov^{*,\pi_{-j}}_{h+1,j}}+H\1\{\a\notin\K_h(s)\}$
            \STATE Set $\ov_{h,j}(s,a_j)=\E_{\a_{-j}\sim\pi_{h,-j}(\cdot|s)}\oq^{*,\pi_{-j}}_{h,j}(s,\a)+b_h(s,a_j,\pi_{h,-j}^s)$
            \STATE Set $\ov_{h,j}^{*,\pi_{-j}}(s)=\proj_{[0,H-h+1]}\Bp{\max_{a_j\in\A_j}\ov_{h,j}(s,a_j)}$
        \ENDFOR
	    \STATE Output $\ov_{1,j}^{*,\pi_{-j}}(s_1)$.
    \end{algorithmic}
\end{algorithm}

\begin{algorithm}[!htbp]
    \caption{\textbf{S}trategy-wise \textbf{B}onus + \textbf{S}urrogate \textbf{M}inimization (\algonamemulti)}
	\label{algo:generalsum markov game}
    \begin{algorithmic}
        \STATE Input: offline dataset $\mathcal{D}$.
        \STATE $\pi^{\mathrm{output}}=\argmin_{\pi\in\Pi}\sum_{j\in[m]}\ov_{1,j}^{*,\pi_{-j}}(s_1)-\uv_{1,j}^\pi(s_1)$, where $\ov_{1,j}^{*,\pi_{-j}}(s_1)$ and $\uv_{1,j}^\pi(s_1)$ are computed via Algorithm \ref{alg:best response estimation} and Algorithm \ref{alg:value estimation}.
	    \STATE Output $\pi^{\mathrm{output}}$.
    \end{algorithmic}
\end{algorithm}

\section{Technical Lemmas}

\subsection{Covering Number of Strategy Classes}\label{apx:covering}
\begin{lemma}\label{lemma:covering number of Pi}
For the no prior knowledge setting ($\Pi=\Pi^{\mathrm{full}}$), we have
$$\log\N(\Pi)=\widetilde{O}\Sp{\sum_{j\in[m]}A_j\log(1/\epsilon_\mathrm{cover})}.$$
\end{lemma}

\begin{proof}
If $\Pi=\Pi^\mathrm{full}$, by Lemma \ref{L-1 covering number} we have
\begin{align*}
    \log\N(\Pi)
    =&\log\Sp{\sum_{s\in\S,h\in[H]}\prod_{j\in[m]}|\C(\Pi_{h,j}(s),\epsilon_\mathrm{cover})|}\\
    =&\log\Sp{SH\prod_{j\in[m]}|\C(\Delta(\A_j),\epsilon_\mathrm{cover})|}\\
    =&\sum_{j\in[m]}\log\Sp{\C(\Delta(\A_j),\epsilon_\mathrm{cover})}+\log(SH)\\
    \leq& \sum_{j\in[m]}A_j\log\Sp{3A_j/\epsilon_\mathrm{cover}}+\log(SH)\tag{Lemma \ref{L-1 covering number}}\\
    =& \widetilde{O}(\sum_{j\in[m]}A_j\log(1/\epsilon_\mathrm{cover})).
\end{align*}
\end{proof}

\begin{lemma}\label{lemma:finite pi}
If $\Pi$ is a finite set, we have
$$\log(\N(\Pi))\leq m\log(|\Pi|)+\log(SH).$$
\end{lemma}

\begin{proof}
We have $|\C(\Pi_{h,j}(s),\epsilon_\mathrm{cover})|\leq |\Pi_{h,j}(s)|\leq|\Pi|$ for all $h\in[H]$ and $j\in[m]$. Plug it into the definition of $\N(\Pi)$ and we can prove the argument. 
\end{proof}

\subsection{Convexity in Two-player Zero-sum Games}\label{apx:convexity}

In this section, we prove that $\uv_h^{\mu_h^s,\nu_h^s}(s)$ is concave and $\ov_h^{\mu_h^s,\nu_h^s}(s)$ is convex for both $\mu_h^s$ and $\nu_h^s$. In addition, we show that \eqref{eq:maxmin epsilon} and \eqref{eq:minmax epsilon} can be achieved efficiently. 

\begin{lemma}\label{lemma:convex}
For any coefficient $c(a_i,b_j)$ s.t. $c(a_i,b_j)\geq0$ for all $a_i\in\A$ and $b_j\in\B$, function $f(\mu,\nu)=\sqrt{\sum_{a_i\in\A,b_j\in\B}c(a_i,b_j)\mu(a_i)^2\nu(b_j)^2}$ defined on $\mu\in\Delta(\A)$ and $\nu\in\Delta(\B)$ is a convex function and $\sqrt{\sum_{a_i\in\A}\sum_{b_j\in\B}c(a_i,b_j)}$-Lipschitz continuous function with respect to $\nu$. In addition, it is convex and $\sqrt{\sum_{a_i\in\A}\sum_{b_j\in\B}c(a_i,b_j)}$-Lipschitz continuous with respect to $\mu$ by symmetry. 
\end{lemma}

\begin{proof}
We use the convention that $\frac{0}{0}=0$. We first compute the first-order derivatives
\begin{equation}\label{eq:gradient}
    \frac{\partial f}{\partial\nu(b_j)}=\frac{\sum_{a_i\in\A}c(a_i,b_j)\mu(a_i)\nu(b_j)^2}{\sqrt{\sum_{a_i\in\A,b\in\B}c(a_i,b)\mu(a_i)^2\nu(b)^2}}.
\end{equation}
By Cauchy-Schwarz inequality, we have
\begin{align*}
    \frac{\partial f}{\partial\nu(b_j)}=&\frac{\sum_{a_i\in\A}c(a_i,b_j)\mu(a_i)\nu(b_j)^2}{\sqrt{\sum_{a_i\in\A,b\in\B}c(a_i,b)\mu(a_i)^2\nu(b)^2}}\\
    \leq&\frac{\sum_{a_i\in\A}c(a_i,b_j)\mu(a_i)\nu(b_j)}{\sqrt{\sum_{a_i\in\A}c(a_i,b_j)\mu(a_i)^2\nu(b_j)^2}}\\
    \leq&\sqrt{\sum_{a_i\in\A}c(a_i,b_j)}.
\end{align*}
Then we have
$$\Norm{\frac{\partial f}{\partial\nu}}_2\leq\sqrt{\sum_{a_i\in\A}\sum_{b_j\in\B}c(a_i,b_j)},$$
which implies $f(\mu,\cdot)$ is $\sqrt{\sum_{a_i\in\A}\sum_{b_j\in\B}c(a_i,b_j)}$-Lipschitz continuous. 

The second-order derivatives are
\begin{align*}
    \frac{\partial^2 f}{\partial \nu(b_j)\partial\nu(b_k)}=-\frac{\Sp{\sum_{a_i\in\A}c(a_i,b_j)\mu(a_i)\nu(b_j)^2}\cdot\Sp{\sum_{a_i\in\A}c(a_i,b_k)\mu(a_i)\nu(b_k)^2}}{\Sp{\sum_{a_i\in\A,b_j\in\B}c(a_i,b_j)\mu(a_i)^2\nu(b_j)^2}^{3/2}},j\neq k,
\end{align*}
\begin{align*}
    \frac{\partial^2 f}{\Sp{\partial \nu(b_j)}^2}=&\frac{\sum_{a_i\in\A}c(a_i,b_j)\nu(b_j)^2}{\sqrt{\sum_{a_i\in\A,b_j\in\B}c(a_i,b_j)\mu(a_i)^2\nu(b_j)^2}}-\frac{\Sp{\sum_{a_i\in\A}c(a_i,b_j)\mu(a_i)\nu(b_j)^2}^2}{\Sp{\sum_{a_i\in\A,b_j\in\B}c(a_i,b_j)\mu(a_i)^2\nu(b_j)^2}^{3/2}}.\\
\end{align*}
Then for arbitrary $x\in\R^B$, we have
\begin{align*}
    &\sum_{j,k\in[B]}x_jx_k\frac{\partial^2 f}{\partial \nu(b_j)\partial\nu(b_k)}\\
    =&\sum_{j\in[B]}\frac{x_j^2\sum_{a_i\in\A}c(a_i,b_j)\nu(b_j)^2}{\sqrt{\sum_{a_i\in\A,b_j\in\B}c(a_i,b_j)\mu(a_i)^2\nu(b_j)^2}}\\
    &-\sum_{j,k\in[B]}\frac{x_jx_k\Sp{\sum_{a_i\in\A}c(a_i,b_j)\mu(a_i)\nu(b_j)^2}\cdot\Sp{\sum_{a_i\in\A}c(a_i,b_k)\mu(a_i)\nu(b_k)^2}}{\Sp{\sum_{a_i\in\A,b_j\in\B}c(a_i,b_j)\mu(a_i)^2\nu(b_j)^2}^{3/2}}\\
    =&\frac{\sum_{j\in[B]}\Sp{x_j^2\sum_{a_i\in\A}c(a_i,b_j)\nu(b_j)^2}\cdot \sum_{a_i\in\A,b_j\in\B}c(a_i,b_j)\mu(a_i)^2\nu(b_j)^2}{\Sp{\sum_{a_i\in\A,b_j\in\B}c(a_i,b_j)\mu(a_i)^2\nu(b_j)^2}^{3/2}}\\
    &-\frac{\Sp{\sum_{j\in[B]}x_j\Sp{\sum_{a_i\in\A}c(a_i,b_j)\mu(a_i)\nu(b_j)^2}}^2}{\Sp{\sum_{a_i\in\A,b_j\in\B}c(a_i,b_j)\mu(a_i)^2\nu(b_j)^2}^{3/2}}\\
    =&\frac{\sum_{j\in[B]}\Sp{x_j^2\sum_{a_i\in\A}c(a_i,b_j)\nu(b_j)^2}\cdot \sum_{a_i\in\A,b_j\in\B}c(a_i,b_j)\mu(a_i)^2\nu(b_j)^2}{\Sp{\sum_{a_i\in\A,b_j\in\B}c(a_i,b_j)\mu(a_i)^2\nu(b_j)^2}^{3/2}}\\
    &-\frac{\Sp{\sum_{j\in[B]}x_j\Sp{\sum_{a_i\in\A}c(a_i,b_j)\mu(a_i)\nu(b_j)^2}}^2}{\Sp{\sum_{a_i\in\A,b_j\in\B}c(a_i,b_j)\mu(a_i)^2\nu(b_j)^2}^{3/2}}.
\end{align*}
By Cauchy-Schwarz's inequality, we have
\begin{align*}
    &\sum_{j\in[B]}\Sp{x_j^2\sum_{a_i\in\A}c(a_i,b_j)\nu(b_j)^2}\cdot \sum_{a_i\in\A,b_j\in\B}c(a_i,b_j)\mu(a_i)^2\nu(b_j)^2\\
    =&\Sp{\sum_{j\in[B]}x_j^2\nu(b_j)^2\sum_{a_i\in\A}c(a_i,b_j)}\cdot \Sp{\sum_{j\in[B]}\nu(b_j)^2\sum_{a_i\in\A}c(a_i,b_j)\mu(a_i)^2}\\
    \geq&\Sp{\sum_{j\in[B]}x_j\nu(b_j)^2\sqrt{\sum_{a_i\in\A}c(a_i,b_j)\sum_{a_i\in\A}c(a_i,b_j)\mu(a_i)^2}}^2\\
    \geq&\Sp{\sum_{j\in[B]}x_j\nu(b_j)^2\sum_{a_i\in\A}c(a_i,b_j)\mu(a_i)}^2\\
    \geq&0. 
\end{align*}
Thus for arbitrary $x\in\R^B$, we have
$$\sum_{j,k\in[B]}x_jx_k\frac{\partial^2 f}{\partial \nu(b_j)\partial\nu(b_k)}\geq0,$$
\end{proof}
which implies $f$ is convex with respect to $\nu$. 

\begin{proposition}\label{prop:convex}
For all $h\in[H]$ and $s\in\S$, $\uv_h^{\mu_h^s,\nu_h^s}$ is concave and $H+H\sqrt{\log(\N(\Pi))\iota}$-Lipschitz with respect to $\mu_h^s$ and $\nu_h^s$. Similarly, $\ov_h^{\mu_h^s,\nu_h^s}$ is convex with respect to $\mu_h^s$ and $\nu_h^s$. As a result, \eqref{eq:maxmin epsilon} and \eqref{eq:minmax epsilon} can be achieved with $(H+H\sqrt{\log(\N(\Pi))\iota})^2/\epsilon_{\mathrm{opt}}^2$ iterations by projected gradient descent.
\end{proposition}

\begin{proof}
Recall that
$$\uv_h^{\mu_h^s,\nu_h^s}(s)=\E_{a\sim\mu_h^s,b\sim\nu_h^s}\uq_h(s,a,b)-H\sqrt{\sum_{(a,b)\in\K_h(s)}\frac{\mu_h^s(a)^2\nu_h^s(b)^2}{n_h(s,a,b)}\log(\N(\Pi))\iota}-\sqrt{\iota}/n.$$
The first term is linear with respect to $\mu_h^s$, The second term is convex by Lemma \ref{lemma:convex} and the last term is a constant. As a result, $\uv_h^{\mu_h^s,\nu_h^s}$ is concave with respect to $\mu_h^s$. By symmetry, it is also concave with respect to $\nu_h^s$. The proof for $\ov_h^{\mu_h^s,\nu_h^s}$ is the same. The Lipschitz constant is a direct implication of Lemma \ref{lemma:convex}. The iteration complexity of projected gradient descent is from Section 3.1 in \citet{bubeck2015convex}. Note that in each iteration we only need to compute the gradient \eqref{eq:gradient} and a projection onto the probability simplex.
\end{proof}

\subsection{Convexity in Multi-player General-sum Games}
In this section, we will show that the bonus $b_h(s,\pi_h^s)$ in multi-player general-sum game is also convex with respect to $\pi_{h,j}^s$ for all $j\in[m]$. 

\begin{lemma}\label{lemma:mpgs convex}
For any $h\in[H]$ and $s\in\S$, $b_h(s,\pi_h^s)$ is convex with respect to $\pi_{h,j}^s$. 
\end{lemma}

\begin{proof}
Recall that
$$b_h(s,\pi^s_h)=H\sqrt{\sum_{\a\in\K_h(s)}\frac{\pi^s_h(\a)^2}{n_h(s,\a)}\log(\N(\Pi))\iota}+\sqrt{\iota}/n.$$

As we have
\begin{align*}
    \sum_{\a\in\K_h(s)}\frac{\pi^s_h(\a)^2}{n_h(s,\a)}= \sum_{a_j\in\A_j}\sum_{\a_{-j}:(a_j,\a_{-j})\in\K_h(s)}\frac{\pi_{h,j}^s(a_j)^2\pi_{h,-j}^s(\a_{-j})^2}{n_h(s,\a)},
\end{align*}
by Lemma \ref{lemma:convex} we have that $b_h(s,\pi_h^s)$ is convex with respect to $\pi_{h,j}^s$. 
\end{proof}

One direction implication is that $\max_{\pi_{h,j}^s}\ov_{h,j}^\pi(s)$ can be achieved by a deterministic strategy $\pi_{h,j}^s\in D(\A_j)$, which will be utilized in Appendix \ref{apx:mpgs}.

\section{Proofs in Section \ref{sec:tpzs}}\label{apx:tpzs}
\begin{lemma}\label{lemma:concentration single}
Fix $h\in[H]$ and $s\in\S$, $\mu'_h(\cdot|s)\in\Delta(\A)$, $\nu'_h(\cdot|s)\in\Delta(\B)$,  with probability $1-\delta$ we have
\begin{align*}
    \abs{\sum_{(a,b)\in\K_h(s)}\mu_h'(a|s)\nu_h'(b|s)\Sp{r_h(s,a,b)+\inner{P_h(s,a,b),\uv_{h+1}}-\widehat{r}_h(s,a,b)-\inner{\widehat{P}_h(s,a,b),\uv_{h+1}}}}\\
    \leq H\sqrt{2\sum_{(a,b)\in\K_h(s)}\frac{\mu_h'(a|s)^2\nu_h'(b|s)^2}{n_h(s,a,b)}\log(2/\delta)},
\end{align*}

\begin{align*}
    \abs{\sum_{(a,b)\in\K_h(s)}\mu_h'(a|s)\nu_h'(b|s)\Sp{r_h(s,a,b)+\inner{P_h(s,a,b),\ov_{h+1}}-\widehat{r}_h(s,a,b)-\inner{\widehat{P}_h(s,a,b),\ov_{h+1}}}}\\
    \leq H\sqrt{2\sum_{(a,b)\in\K_h(s)}\frac{\mu_h'(a|s)^2\nu_h'(b|s)^2}{n_h(s,a,b)}\log(2/\delta)}.
\end{align*}

\end{lemma}

\begin{proof}
We use $k_h^i(s,a,b)$ to denote the index of $(s,a,b)$ appears in the dataset at timestep $h$ for $i$th time. We prove the first argument and the second argument holds similarly. With probability $1-\delta$, we have
\begin{align*}
    &\abs{\sum_{(a,b)\in\K_h(s)}\mu'_h(a|s)\nu'_h(b|s)\Sp{r_h(s,a,b)+\inner{P_h(s,a,b),\uv_{h+1}}-\widehat{r}_h(s,a,b)-\inner{\widehat{P}_h(s,a,b),\uv_{h+1}}}}\\
    =&\left|\sum_{(a,b)\in\K_h(s)}\sum_{i=1}^{n_h(s,a,b)}\frac{\mu'_h(a|s)\nu'_h(b|s)}{n_h(s,a,b)}\Sp{r_h^{k_h^i(s,a,b)}-r_h(s,a,b)}\right.\\&+\left.\sum_{(a,b)\in\K_h(s)}\sum_{i=1}^{n_h(s,a,b)}\frac{\mu'_h(a|s)\nu'_h(b|s)}{n_h(s,a,b)}\Sp{\uv_{h+1}(s_{h+1}^{k_h^i(s,a,b)})-\inner{P_h(s,a,b),\uv_{h+1}}}\right|\\
    \leq& \sqrt{\frac{1}{2}\sum_{(a,b)\in\K_h(s)}\frac{\mu'_h(a|s)^2\nu'_h(b|s)^2}{n_h(s,a,b)}\log(2/\delta)}+H\sqrt{\frac{1}{2}\sum_{(a,b)\in\K_h(s)}\frac{\mu'_h(a|s)^2\nu'_h(b|s)^2}{n_h(s,a,b)}\log(2/\delta)}\\
    \leq& H\sqrt{2\sum_{(a,b)\in\K_h(s)}\frac{\mu'_h(a|s)^2\nu'_h(b|s)^2}{n_h(s,a,b)}\log(2/\delta)},
\end{align*}
where the first inequality is from Hoeffding's inequality and the fact that $\uv_{h+1}$ has no dependence on the dataset at timestep $h$.
\end{proof}

\begin{lemma}\label{lemma:concentration all}
With probability $1-\delta$, for all $h\in[H],s\in\S,\mu_h^s\in \Pi^{\mathrm{max}}_h(s)$, $\nu_h^s\in D(\B)$, we have
\begin{align*}
    \abs{\sum_{(a,b)\in\K_h(s)}\mu_h^s(a)\nu_h^s(b)\Sp{r_h(s,a,b)+\inner{P_h(s,a,b),\uv_{h+1}}-\widehat{r}_h(s,a,b)-\inner{\widehat{P}_h(s,a,b),\uv_{h+1}}}}\\\leq b_h(s,\mu_h^s,\nu_h^s),
\end{align*}
and for $\mu_h^s\in D(\A)$, $\nu_h^s\in \Pi^{\mathrm{min}}_h(s)$, we have
\begin{align*}
    \abs{\sum_{(a,b)\in\K_h(s)}\mu_h^s(a)\nu_h^s(b)\Sp{r_h(s,a,b)+\inner{P_h(s,a,b),\ov_{h+1}}-\widehat{r}_h(s,a,b)-\inner{\widehat{P}_h(s,a,b),\ov_{h+1}}}}\\\leq b_h(s,\mu_h^s,\nu_h^s).
\end{align*}
Denote this event as $\G$.
\end{lemma}

\begin{proof}
We prove the first argument and the second argument holds similarly. 
First, using a union bound for all $h\in[H],s\in\S,\mu'^s_h\in \C (\Pi^{\mathrm{max}}_h(s))$, $\nu'^s_h\in D(\B)$ on Lemma \ref{lemma:concentration single}, with probability $1-\delta$, we have
\begin{align*}
    &\abs{\sum_{(a,b)\in\K_h(s)}\mu'^s_h(a)\nu'^s_h(b)\Sp{r_h(s,a,b)+\inner{P_h(s,a,b),\uv_{h+1}}-\widehat{r}_h(s,a,b)-\inner{\widehat{P}_h(s,a,b),\uv_{h+1}}}}\\
    \leq& H\sqrt{2\sum_{(a,b)\in\K_h(s)}\frac{\mu'^s_h(a)^2\nu'^s_h(b)^2}{n_h(s,a,b)}\log(2\sum_{s\in\S,h\in[H]}|(\C (\Pi^{\mathrm{max}}_h(s))|B+|\C (\Pi^{\mathrm{min}}_h(s))|A)/\delta)}\\
    \leq& H\sqrt{2\sum_{(a,b)\in\K_h(s)}\frac{\mu'^s_h(a)^2\nu'^s_h(b)^2}{n_h(s,a,b)}\log(2\N(\Pi)ABSH\delta)}.\tag{See Definition \ref{def:N}}
\end{align*}

Note that $r_h(s,a,b)+\inner{P_h(s,a,b),\uv_{h+1}}-\widehat{r}_h(s,a,b)-\inner{\widehat{P}_h(s,a,b),\uv_{h+1}}$ is bounded in $[-H,H]$ as $r_h(s,a,b)\in[0,1]$ and $\uv_{h+1}\in[0,H-h]$. For any $\mu_h(\cdot|s)\in \Pi^{\mathrm{max}}_h(s)$ and $\nu_h(\cdot|s)\in D(\B)$, there exists $\mu'_h(\cdot|s)\in\C (\Pi^{\mathrm{max}}_h(s))$ and $\nu'_h(\cdot|s)\in D(\B)$ such that $\|\mu_h(\cdot|s)-\mu'_h(\cdot|s)\|\leq\epsilon_{\mathrm{cover}}$ and $\|\nu_h(\cdot|s)-\nu'_h(\cdot|s)\|=0\leq\epsilon_{\mathrm{cover}}$. So with Lemma \ref{lemma:net error}, we have
\begin{align*}
    &\left|\sum_{(a,b)\in\K_h(s)}\mu'_h(a|s)\nu'_h(b|s)\Sp{r_h(s,a,b)+\inner{P_h(s,a,b),\uv_{h+1}}-\widehat{r}_h(s,a,b)-\inner{\widehat{P}_h(s,a,b),\uv_{h+1}}}\right.\\
    &-\left.\sum_{(a,b)\in\K_h(s)}\mu_h(a|s)\nu_h(b|s)\Sp{r_h(s,a,b)+\inner{P_h(s,a,b),\uv_{h+1}}-\widehat{r}_h(s,a,b)-\inner{\widehat{P}_h(s,a,b),\uv_{h+1}}}\right|\\
    \leq& 2\epsilon_{\mathrm{cover}} H.
\end{align*}

By Lemma \ref{lemma:net error 2}, we have
\begin{align*}
    \abs{\sqrt{\sum_{(a,b)\in\K_h(s)}\frac{\mu'_h(a|s)^2\nu'_h(b|s)^2}{n_h(s,a,b)}}-\sqrt{\sum_{(a,b)\in\K_h(s)}\frac{\mu_h(a|s)^2\nu_h(b|s)^2}{n_h(s,a,b)}}}\leq 2\sqrt{\epsilon_{\mathrm{cover}}}.
\end{align*}

Combining all these parts together and then with probability $1-\delta$, we have
\begin{align*}
    &\abs{\sum_{(a,b)\in\K_h(s)}\mu_h(a|s)\nu_h(b|s)\Sp{r_h(s,a,b)+\inner{P_h(s,a,b),\uv_{h+1}}-\widehat{r}_h(s,a,b)-\inner{\widehat{P}_h(s,a,b),\uv_{h+1}}}}\\
    \leq& H\sqrt{2\sum_{(a,b)\in\K_h(s)}\frac{\mu_h(a|s)^2\nu_h(b|s)^2}{n_h(s,a,b)}\log(2\N(\Pi,\epsilon_{\mathrm{cover}})ABSH/\delta)}+2\epsilon_{\mathrm{cover}} H\\
    &+2H\sqrt{2\epsilon_{\mathrm{cover}} \log(2\N(\Pi,\epsilon_{\mathrm{cover}})ABSH/\delta)}.
\end{align*}

Set $\epsilon_{\mathrm{cover}}=\frac{1}{(A+B)H^2n^2}$ and with some algebra we can get
\begin{align*}
    &\abs{\sum_{(a,b)\in\K_h(s)}\mu_h(a|s)\nu_h(b|s)\Sp{r_h(s,a,b)+\inner{P_h(s,a,b),\uv_{h+1}}-\widehat{r}_h(s,a,b)-\inner{\widehat{P}_h(s,a,b),\uv_{h+1}}}}\\
    \leq& H\sqrt{2\sum_{(a,b)\in\K_h(s)}\frac{\mu_h(a|s)^2\nu_h(b|s)^2}{n_h(s,a,b)}\log(2\N(\Pi)ABSHn/\delta)}+\sqrt{32\log(2ABSHn/\delta)}/n\\
    \leq& H\sqrt{\sum_{(a,b)\in\K_h(s)}\frac{\mu_h(a|s)^2\nu_h(b|s)^2}{n_h(s,a,b)}\log(\N(\Pi))\iota}+\sqrt{\iota}/n,
\end{align*}
where $\iota=32\log(2ABSHn/\delta)$. 
\end{proof}

\begin{lemma}\label{lemma:pessimism policy}
Under event $\G$, for all $s\in\S$, $h\in[H]$, $\mu_h(\cdot|s)\in \Pi^{\mathrm{max}}_h(s)$ and $\nu_h(\cdot|s)\in D(\B)$, we have
$$\uv_h^{\mu_h^s,\nu_h^s}(s)\leq \E_{a\sim\mu_h(\cdot|s),b\sim\nu_h(\cdot|s)}\Mp{r_h(s,a,b)+\inner{P_h(s,a,b),\uv_{h+1}}},$$
and for $\mu_h^s\in D(\A)$, $\nu_h^s\in \Pi^{\mathrm{min}}_h(s)$, we have
$$\ov_h^{\mu_h^s,\nu_h^s}(s)\geq \E_{a\sim\mu_h(\cdot|s),b\sim\nu_h(\cdot|s)}\Mp{r_h(s,a,b)+\inner{P_h(s,a,b),\ov_{h+1}}}.$$
\end{lemma}

\begin{proof}
Under the good event $\G$, we have
\begin{align*}
&\uv_h^{\mu_h^s,\nu_h^s}(s)\\
    =&\E_{a\sim\mu_h(\cdot|s),b\sim\nu_h(\cdot|s)}\uq_h(s,a,b)-b_h(s,\mu_h^s,\nu_h^s)\\
    =&\sum_{(a,b)\in\K_h(s)}\mu_h(a|s)\nu_h(b|s)\Sp{\widehat{r}_h(s,a,b)+\inner{\widehat{P}_h(s,a,b),\uv_{h+1}}}-b_h(s,\mu_h^s,\nu_h^s)\\
    \leq& \sum_{(a,b)\in\K_h(s)}\mu_h(a|s)\nu_h(b|s)\Sp{r_h(s,a,b)+\inner{P_h(s,a,b),\uv_{h+1}}}\tag{Lemma \ref{lemma:concentration all}}\\
    \leq& \sum_{a\in\A,b\in\B}\mu_h(a|s)\nu_h(b|s)\Sp{r_h(s,a,b)+\inner{P_h(s,a,b),\uv_{h+1}}}\tag{$\uv_{h+1}\geq0$}\\
    =&\E_{a\sim\mu_h(\cdot|s),b\sim\nu_h(\cdot|s)}\Mp{r_h(s,a,b)+\inner{P_h(s,a,b),\uv_{h+1}}}. 
\end{align*}
Similarly we have
\begin{align*}
&\ov_h^{\mu_h^s,\nu_h^s}(s)\\
    =&\E_{a\sim\mu_h(\cdot|s),b\sim\nu_h(\cdot|s)}\oq_h(s,a,b)+b_h(s,\mu_h^s,\nu_h^s)\\
    =&\sum_{(a,b)\in\K_h(s)}\mu_h(a|s)\nu_h(b|s)\Sp{\widehat{r}_h(s,a,b)+\inner{\widehat{P}_h(s,a,b),\ov_{h+1}}}+H\sum_{(a,b)\notin\K_h(s)}\mu_h(a|s)\nu_h(b|s)\\&+b_h(s,\mu_h^s,\nu_h^s)\\
    \geq& \sum_{(a,b)\in\K_h(s)}\mu_h(a|s)\nu_h(b|s)\Sp{r_h(s,a,b)+\inner{P_h(s,a,b),\ov_{h+1}}}+H\sum_{(a,b)\notin\K_h(s)}\mu_h(a|s)\nu_h(b|s)\tag{Lemma \ref{lemma:concentration all}}\\
    \geq& \sum_{a\in\A,b\in\B}\mu_h(a|s)\nu_h(b|s)\Sp{r_h(s,a,b)+\inner{P_h(s,a,b),\ov_{h+1}}}\tag{$\ov_{h+1}\leq H-h$}\\
    =&\E_{a\sim\mu_h(\cdot|s),b\sim\nu_h(\cdot|s)}\Mp{r_h(s,a,b)+\inner{P_h(s,a,b),\ov_{h+1}}}. 
\end{align*}
\end{proof}

\begin{lemma}\label{lemma:pessimism state}
Under event $\G$, for all $s\in\S$ and $h\in[H]$, with probability $1-\delta$, we have
$$\uv_h(s)\leq V_h^{\um,*}(s),\ov_h(s)\geq V_h^{*,\on}(s).$$
\end{lemma}

\begin{proof}
We prove the first argument and the second argument holds similarly. We prove this argument by induction. It holds trivially for $h=H+1$ as both sides are equal to zero. Suppose the argument holds for timestep $h+1$. Then for any $s\in\S$, we have
\begin{align*}
    \uv_h(s)=&\proj_{[0,H-h+1]}\Bp{\uv_h^{\um_h^s,\un_h^s}(s)}\\
    =&\proj_{[0,H-h+1]}\Bp{\min_{\nu^s_h\in D(\B)}\uv_h^{\um^s_h,\nu^s_h}(s)}\\
    \leq& \proj_{[0,H-h+1]}\Bp{\min_{\nu^s_h\in D(\B)}\E_{a\sim\mu_h(\cdot|s),b\sim\nu_h(\cdot|s)}\Mp{r_h(s,a,b)+\inner{P_h(s,a,b),\uv_{h+1}}}}\tag{Lemma \ref{lemma:pessimism policy}}\\
    \leq& \proj_{[0,H-h+1]}\Bp{\min_{\nu^s_h\in D(\B)}\E_{a\sim\mu_h(\cdot|s),b\sim\nu_h(\cdot|s)}\Mp{r_h(s,a,b)+\inner{P_h(s,a,b),V^{\um,*}_{h+1}}}}\tag{Induction hypothesis}\\
    =& \proj_{[0,H-h+1]}\Bp{V_h^{\um,*}(s)}\tag{There always exists a best response in $D(\B)$}\\
    =& V_h^{\um,*}(s).
\end{align*}
By induction, the argument holds for all $h\in[H]$. The proof for $\ov_h(s)$ is the same. 
\end{proof}

For any $\mu_h^s\in \Delta(\A)$, with a slight abuse of notation, we define
$$\un_h^s(\mu_h^s):=\argmin_{\nu_h^s\in D(\B)}\uv_h^{\mu_h^s,\nu_h^s}. $$
Note that $\un_h^s=\un_h^s(\um_h^s)$. We use $\un(\mu)\in \Pi^{\mathrm{min,det}}$ to denote a strategy for player 2 such that she use $\un_h^s(\mu_h^s)$ at state $s$ and timestep $h$. 

\begin{lemma}\label{lemma:gap decomp}
Under the good event $\G$, for any $\widetilde{\mu}\in \Pi^{\mathrm{max}}$ and $\widetilde{\nu}\in \Pi^{\mathrm{min}}$, we have
$$V_1^{\widetilde{\mu},*}(s_1)-V_1^{\um,*}(s_1)\leq\E_{\widetilde{\mu},\un(\widetilde{\mu})}\sum_{h=1}^H\widehat{b}_h(s_h,\widetilde{\mu}_h^{s_h},\un^{s_h}_h(\widetilde{\mu}^{s_h}_h))+H\epsilon_{\mathrm{opt}},$$
$$V_1^{*,\on}(s_1)-V_1^{*,\widetilde{\nu}}(s_1)\leq\E_{\om(\widetilde{\nu}),\widetilde{\nu}}\sum_{h=1}^H\widehat{b}_h(s_h,\om^{s_h}_h(\widetilde{\nu}^{s_h}_h),\widetilde{\nu}^{s_h}_h)+H\epsilon_{\mathrm{opt}}.$$
\end{lemma}

\begin{proof}
We prove the first argument and the second argument holds similarly. By Lemma \ref{lemma:pessimism state}, we have
$$V_1^{\widetilde{\mu},*}(s_1)-V_1^{\um,*}(s_1)\leq V_1^{\widetilde{\mu},*}(s_1)-\uv_1(s_1).$$
Now we work on the difference between the NE value and the pessimistic estimate.
\begin{align*}
    &V_1^{\widetilde{\mu},*}(s_1)-\uv_1(s_1)\\
    =& \min_{\nu^{s_1}_1}\E_{\widetilde{\mu}^{s_1}_1,\nu^{s_1}_1}Q_1^{\widetilde{\mu},*}(s_1,a_1,b_1)-\proj_{[0,H]}\Bp{\uv_1^{\um^{s_1}_1,\un^{s_1}_1}(s_1)}\\
    \leq& \min_{\nu^{s_1}_1}\E_{\widetilde{\mu}^{s_1}_1,\nu^{s_1}_1}Q_1^{\widetilde{\mu},*}(s_1,a_1,b_1)-\uv_1^{\um^{s_1}_1,\un^{s_1}_1}(s_1)\tag{$\uv_1^{\um^{s_1}_1,\un^{s_1}_1}(s_1)\leq H$ by \eqref{eq:uq} and \eqref{eq:uv}}\\
    \leq& \E_{\widetilde{\mu}^{s_1}_1,\un^{s_1}_1(\widetilde{\mu}^{s_1}_1)}Q_1^{\widetilde{\mu},*}(s_1,a_1,b_1)-\uv_1^{\widetilde{\mu}^{s_1}_1,\un^{s_1}_1(\widetilde{\mu}_1)}(s_1)+\epsilon_{\mathrm{opt}}\\
    =& \E_{\widetilde{\mu}^{s_1}_1,\un^{s_1}_1(\widetilde{\mu}_1)}\Mp{Q_1^{\widetilde{\mu},*}(s_1,a_1,b_1)-\uq_1(s_1,a_1,b_1)}+b_1(s_1,\widetilde{\mu}^{s_1}_1,\un^{s_1}_1(\widetilde{\mu}^{s_1}_1))+\epsilon_{\mathrm{opt}}\\
    =& \E_{\widetilde{\mu}^{s_1}_1,\un^{s_1}_1(\widetilde{\mu}^{s_1}_1)}\Mp{r_1(s_1,a_1,b_1)+\inner{P_1(s_1,a_1,b_1),V_2^{\widetilde{\mu},*}}-\widehat{r}_1(s_1,a_1,b_1)-\inner{\widehat{P}_1(s_1,a_1,b_1),\uv_2}}\\
    &+b_1(s_1,\widetilde{\mu}^{s_1}_1,\un^{s_1}_1(\widetilde{\mu}^{s_1}_1))+\epsilon_{\mathrm{opt}}\\
    \leq& \E_{\widetilde{\mu}^{s_1}_1,\un^{s_1}_1(\widetilde{\mu}^{s_1}_1)}\Mp{V_2^{\widetilde{\mu},*}(s_2)-\uv_2(s_2)}+2b_1(s_1,\widetilde{\mu}^{s_1}_1,\un^{s_1}_1(\widetilde{\mu}^{s_1}_1))\\
    &+H\sum_{(a_1,b_1)\notin\K_1(s_1)}\widetilde{\mu}^{s_1}_1(a_1)\un^{s_1}_1(\widetilde{\mu}^{s_1}_1)(b_1)+\epsilon_{\mathrm{opt}}\tag{Lemma \ref{lemma:concentration all}}\\
    \leq& \E_{\widetilde{\mu},\un(\widetilde{\mu})}\sum_{h=1}^H\Sp{2b_h(s_h,\widetilde{\mu}^{s_h}_h,\un_h^{s_h}(\widetilde{\mu}^{s_h}_h))+H\sum_{(a_h,b_h)\notin\K_h(s_h)}\widetilde{\mu}^{s_h}_h(a_h)\un^{s_h}_h(\widetilde{\mu}^{s_h}_h)(b_h)}+H\epsilon_{\mathrm{opt}},
\end{align*}
where the last inequality is from telescoping from $h=1$ to $h=H$.
\end{proof}

\begin{proposition}
Under the good event $\G$, we have
\begin{align*}
    &\mathrm{Gap}(\pi^{\mathrm{output}}) \leq\\
    &  \min_{\pi=(\mu,\nu)\in\Pi}\max_{\pi'=(\mu',\nu')\in \Pi^{\mathrm{det}}}\left[\mathrm{Gap}(\pi)+\E_{\mu,\nu'}\sum_{h=1}^H\widehat{b}_h(s_h,\mu^{s_h}_h,\nu'^{s_h}_h)\right.\left.+\E_{\mu',\nu}\sum_{h=1}^H\widehat{b}_h(s_h,\mu'^{s_h}_h,\nu^{s_h}_h)\right]+2H\epsilon_{\mathrm{opt}}.
\end{align*}

\end{proposition}

\begin{proof}
This is a direct deduction of Lemma \ref{lemma:gap decomp}. Note that $(\un(\widetilde{\mu}),\om(\widetilde{\nu}))\in \Pi^{\mathrm{det}}$. 
\end{proof}

\subsection{Dataset-dependent Bound}

\begin{lemma}\label{lemma:bound on sumb dataset}
Suppose $\widehat{C}(\mu,\nu)$ is finite. For any $h\in[H]$ and strategy $\mu'$ and $\nu'$, we have
$$\E_{\mu,\nu'}b_h(s_h,\mu^{s_h}_h,\nu'^{s_h}_h)\leq 2H\sqrt{S\log(\N(\Pi))\widehat{C}(\mu,\nu)\iota/n},$$
$$\E_{\mu',\nu}b_h(s_h,\mu'^{s_h}_h,\nu^{s_h}_h)\leq 2H\sqrt{S\log(\N(\Pi))\widehat{C}(\mu,\nu)\iota/n}.$$
\end{lemma}

\begin{proof}
We prove the first argument and the second argument holds similarly.
\begin{align*}
&\E_{\mu,\nu'}b_h(s_h,\mu^{s_h}_h,\nu'^{s_h}_h)\\
    =&\E_{\mu,\nu'}\Mp{H\sqrt{\sum_{(a,b)\in\K_h(s)}\frac{\mu^{s_h}_h(a)^2\nu'^{s_h}_h(b)^2}{n_h(s,a,b)}\log(\N(\Pi))\iota}+\frac{\sqrt{\iota}}{n}}\\
    =&\sum_{s_h\in\S}H\sqrt{\log(\N(\Pi))\iota}\sqrt{\sum_{(a_h,b_h)\in\K_h(s_h)}\frac{d_h^{\mu,\nu'}(s_h,a_h,b_h)^2}{n_h(s_h,a_h,b_h)}}+\frac{\sqrt{\iota}}{n}\\
    =& \sum_{s_h\in\S}H\sqrt{\log(\N(\Pi))\iota}\sqrt{\sum_{(a_h,b_h)\in\K_h(s_h)}\frac{d_h^{\mu,\nu'}(s_h,a_h,b_h)^2}{n\cdot \widehat{d}_h(s_h,a_h,b_h)}}+\frac{\sqrt{\iota}}{n}\\
    \leq& \sum_{s_h\in\S}H\sqrt{\log(\N(\Pi))\iota}\sqrt{\sum_{(a_h,b_h)\in\K_h(s_h)}d_h^{\mu,\nu'}(s_h,a_h,b_h)\widehat{C}(\mu,\nu)/n}+\frac{\sqrt{\iota}}{n}\\
    \leq& H\sqrt{S\log(\N(\Pi))\widehat{C}(\mu,\nu)\iota/n}+\frac{\sqrt{\iota}}{n}\\
    \leq& 2H\sqrt{S\log(\N(\Pi))\widehat{C}(\mu,\nu)\iota/n}.
\end{align*}
\end{proof}

\begin{lemma}\label{lemma:bound on sump dataset}
Suppose $\widehat{C}(\mu,\nu)$ is finite. For any $h\in[H]$ and strategy $\mu'$ and $\nu'$, we have
$$\E_{\mu,\nu'}\sum_{(a_h,b_h)\notin\K_h(s_h)}\mu^{s_h}_h(a_h)\nu'^{s_h}_h(b_h)=0,$$
$$\E_{\mu',\nu}\sum_{(a_h,b_h)\notin\K_h(s_h)}\mu'^{s_h}_h(a_h)\nu^{s_h}_h(b_h)=0.$$
\end{lemma}

\begin{proof}
We prove the first argument and the second argument holds similarly.
\begin{align*}
    &\E_{\mu,\nu'}\sum_{(a_h,b_h)\notin\K_h(s_h)}\mu^{s_h}_h(a_h)\nu'^{s_h}_h(b_h)\\
    =&\E_{\mu,\nu'}\sum_{(a_h,b_h):\widehat{d}_h(s_h,a_h,b_h)=0}\mu^{s_h}_h(a_h)\nu'^{s_h}_h(b_h)\\
    =&\sum_{(a_h,b_h):\widehat{d}_h(s_h,a_h,b_h)=0}d_h^{\mu,\nu'}(s_h,a_h,b_h)\\
    \leq&\sum_{(a_h,b_h):\widehat{d}_h(s_h,a_h,b_h)=0}C(\mu,\nu')\widehat{d}_h(s_h,a_h,b_h)\\
    =&0. 
\end{align*}
\end{proof}

\begin{lemma}\label{lemma:bhat bound dataset}
For any strategy $(\mu,\nu)\in\Pi$, we have
\begin{align*}
    &\max_{\nu'\in  \Pi^{\mathrm{min,det}}}\E_{\mu,\nu'}\sum_{h=1}^H\widehat{b}_h(s_h,\mu^{s_h}_h,\nu'^{s_h}_h)+\max_{\mu'\in  \Pi^{\mathrm{max,det}}}\E_{\mu',\nu}\sum_{h=1}^H\widehat{b}_h(s_h,\mu'^{s_h}_h,\nu^{s_h}_h)\\&\leq 4H^2\sqrt{S\log(|\N(\Pi)|)\widehat{C}(\mu,\nu)\iota/n}. 
\end{align*}
\end{lemma}

\begin{proof}
If $\widehat{C}(\mu,\nu)$ is infinite, the argument holds immediately. Otherwise we can prove it by Lemma \ref{lemma:bound on sumb dataset} and Lemma \ref{lemma:bound on sump dataset}. 
\end{proof}

\begin{theorem}\label{thm:dataset}
Suppose $\pi^{\mathrm{output}}$ is the output of Algorithm \ref{algo:zerosum markov game}. With probability $1-\delta$,  we have
$$\mathrm{Gap}(\pi^{\mathrm{output}})\leq\min_{\pi=(\mu,\nu)\in\Pi}\Mp{\mathrm{Gap}(\pi)+4H^2\sqrt{S\log(\N(\Pi))\widehat{C}(\pi)\iota/n}}.$$
\end{theorem}

\begin{proof}
This can be derived from Lemma \ref{lemma:gap decomp}, Lemma \ref{lemma:bhat bound dataset}  directly. 
\end{proof}

\subsection{Dataset-independent Bound}

\begin{lemma}\label{lemma:warmup}
With probability $1-\delta$, for all $h,s,a,b$, we have
$$n_h(s,a,b)\geq \Sp{1-\sqrt{\frac{2\log(SABH/\delta)}{np_\mathrm{min}}}}nd_h(s,a,b). $$
As a result, if $n\geq \frac{8\log(SABH/\delta)}{p_\mathrm{min}}$, for any strategy $\pi$ we have
$$2C(\pi)\geq\widehat{C}(\pi). $$
\end{lemma}

\begin{proof}
For a fixed $s,a,b,h$, for any $\epsilon>0$ we have
$$\P(n_h(s,a,b)<(1-\epsilon) nd_h(s,a,b))\leq\exp\Sp{-\frac{\epsilon^2nd_h(s,a,b)}{2}}\leq\exp\Sp{-\frac{\epsilon^2np_\mathrm{min}}{2}}.$$
With a union bound, we have
$$\P(\exists h,s,a,b: \P(n_h(s,a,b)<(1-\epsilon) nd_h(s,a,b)))\leq SABH\exp\Sp{-\frac{\epsilon^2np_\mathrm{min}}{2}}.$$
The RHS is smaller than $\delta$ if we set
$$\epsilon=\sqrt{\frac{2\log(SABH/\delta)}{np_\mathrm{min}}}.$$
If $n\geq \frac{8\log(SABH/\delta)}{p_\mathrm{min}}$, we have
$$\widehat{d}_h(s,a,b)=\frac{n_h(s,a,b)}{n}\geq\frac{d_h(s,a,b)}{2}. $$
By Definition \ref{def:marl dataset} and Definition \ref{def:policy}, we have
$$2C(\pi)\geq\widehat{C}(\pi).$$
\end{proof}

The following Lemma is from Lemma A.1 in \citet{xie2021policy}. For completeness we provide a proof here. 

\begin{lemma}\label{lemma:warmup2}
With probability at least $1-\delta$, for all $h\in[H]$, $s\in\S$, $a\in\A$ and $b\in\B$, we have
$$n_h(s,a,b)\vee1\geq \frac{nd_h(s,a,b)}{\iota}.$$
\end{lemma}

\begin{proof}
For fixed $h\in[H]$, $s\in\S$, $a\in\A$ and $b\in\B$, $n_h(s,a,b)$ is a binomial random variable following $\mathrm{Bin}(n,d_h(s,a,b))$. We show that with probability $1-\delta$, we have
$$n_h(s,a,b)\vee1\geq \frac{nd_h(s,a,b)}{8\log(1/\delta)}.$$
If $d_h(s,a,b)\leq 8\log(1/\delta)/n$, the argument holds directly. Otherwise by the multiplicative Chernoff bound, we have
$$P(n_h(s,a,b)<nd_h(s,a,b)/2)\leq\exp(-nd_h(s,a,b)/8)\leq\delta. $$
So with probability $1-\delta$, we have $n_h(s,a,b)\geq nd_h(s,a,b)/2\geq nd_h(s,a,b)/8\log(1/\delta)$. Then with union bound we can prove the lemma. 
\end{proof}

\begin{lemma}\label{lemma:bound on sumb policy}
With probability $1-\delta$ for any $h\in[H]$ we have
$$\E_{\mu,\nu'}b_h(s_h,\mu^{s_h}_h,\nu'^{s_h}_h)\leq 2H\sqrt{S\log(\N(\Pi))C(\mu,\nu)\iota^2/n},$$
$$\E_{\mu',\nu}b_h(s_h,\mu'^{s_h}_h,\nu^{s_h}_h)\leq 2H\sqrt{S\log(\N(\Pi))C(\mu,\nu)\iota^2/n}.$$
\end{lemma}

\begin{proof}
From Lemma \ref{lemma:warmup2}, with probability $1-\delta$, for all $h,s,a,b$, we have
$$n_h(s,a,b)\vee1\geq \frac{nd_h(s,a,b)}{\iota}.$$
For $(a,b)\in\K_h(s)$, we have $n_h(s,a,b)\geq 1$ and thus $n_h(s,a,b)\geq \frac{nd_h(s,a,b)}{\iota}$. 
\begin{align*}
&\E_{\mu,\nu'}b_h(s_h,\mu^{s_h}_h,\nu'^{s_h}_h)\\
    =&\E_{\mu,\nu'}\Mp{H\sqrt{\sum_{(a,b)\in\K_h(s)}\frac{\mu^{s_h}_h(a)^2\nu'^{s_h}_h(b)^2}{n_h(s,a,b)}\log(\N(\Pi))\iota}+\frac{\sqrt{\iota}}{n}}\\
    =&\sum_{s_h\in\S}H\sqrt{\log(\N(\Pi))\iota}\sqrt{\sum_{(a_h,b_h)\in\K_h(s_h)}\frac{d_h^{\mu,\nu'}(s_h,a_h,b_h)^2}{n_h(s_h,a_h,b_h)}}+\frac{\sqrt{\iota}}{n}\\
    =& \sum_{s_h\in\S}H\sqrt{\log(\N(\Pi))\iota^2}\sqrt{\sum_{(a_h,b_h)\in\K_h(s_h)}\frac{d_h^{\mu,\nu'}(s_h,a_h,b_h)^2}{n\cdot d_h(s_h,a_h,b_h)}}+\frac{\sqrt{\iota}}{n}\\
    \leq& \sum_{s_h\in\S}H\sqrt{\log(\N(\Pi))\iota^2}\sqrt{\sum_{(a_h,b_h)\in\K_h(s_h)}d_h^{\mu^*,\un(\mu^*)}(s_h,a_h,b_h)C^*/n}+\frac{\sqrt{\iota}}{n}\\
    \leq& H\sqrt{S\log(\N(\Pi))C^*\iota^2/n}+\frac{\sqrt{\iota}}{n}\\
    \leq& 2H\sqrt{S\log(\N(\Pi))C^*\iota^2/n}.
\end{align*}
\end{proof}

\begin{lemma}\label{lemma:bound on sump policy}
With probability $1-\delta$ for any $\mu'\in \Pi^{\mathrm{max,det}}$, $\nu'\in \Pi^\mathrm{min,det}$, $h\in[H]$ and $t\in[0,1]$ we have
$$\E_{\mu,\nu'}\sum_{(a_h,b_h)\notin\K_h(s_h)}\mu_h^{s_h}(a_h)\nu'^{s_h}_h(b_h)\leq \Sp{SAC(\mu,\nu)\iota/n}^{t},$$
$$\E_{\mu',\nu}\sum_{(a_h,b_h)\notin\K_h(s_h)}\mu'^{s_h}_h(a_h)\nu^{s_h}_h(b_h)\leq \Sp{SBC(\mu,\nu)\iota/n}^{t}.$$
In addition, if $\mu\in \Pi^{\mathrm{max,det}}$ and $\nu\in \Pi^\mathrm{min,det}$, we have
$$\E_{\mu,\nu'}\sum_{(a_h,b_h)\notin\K_h(s_h)}\mu_h^{s_h}(a_h)\nu'^{s_h}_h(b_h)\leq \Sp{SC(\mu,\nu)\iota/n}^{t},$$
$$\E_{\mu',\nu}\sum_{(a_h,b_h)\notin\K_h(s_h)}\mu'^{s_h}_h(a_h)\nu^{s_h}_h(b_h)\leq \Sp{SC(\mu,\nu)\iota/n}^{t}.$$
\end{lemma}

\begin{proof}
We prove the first argument and the second one holds similarly. From Lemma \ref{lemma:warmup2}, with probability $1-\delta$, for all $h,s,a,b$, we have
$$n_h(s,a,b)\vee1\geq \frac{nd_h(s,a,b)}{\iota}.$$
For $(a,b)\notin\K_h(s)$, we have $n_h(s,a,b)=0$ and thus $\iota\geq nd_h(s,a,b)$. 
Then for any $t\in[0,1]$, we have
\begin{align*}
    &\E_{\mu,\nu'}\sum_{(a_h,b_h)\notin\K_h(s_h)}\mu_h^{s_h}(a_h)\nu'^{s_h}_h(b_h)\\
    \leq&\E_{\mu,\nu'}\sum_{(a_h,b_h)\in\A\times\B}\frac{\mu^{s_h}_h(a_h)\nu'^{s_h}_h(b_h)\iota^t}{(nd_h(s_h,a_h,b_h))^t}\\
    =&\sum_{s_h\in\S}\sum_{a_h\in\A,b_h=\nu'_h(s_h)}\frac{d_h^{\mu,\nu'}(s_h,a_h,b_h)\iota^t}{(nd_h(s_h,a_h,b_h))^t}\\
    \leq &\sum_{s_h\in\S}\sum_{a_h\in\A,b_h=\nu'_h(a_h)}\frac{C^t(\mu,\nu)\iota^t}{n^t}\Sp{d_h^{\mu,\nu'}(s_h,a_h,b_h)}^{1-t}\\
    \leq&\Sp{SAC(\mu,\nu)\iota/n}^{t}.\tag{Cauchy-Schwarz Inequality}
\end{align*}
If we have $\mu\in M^{\mathrm{det}}$, then we have
\begin{align*}
    &\E_{\mu,\nu'}\sum_{(a_h,b_h)\notin\K_h(s_h)}\mu_h^{s_h}(a_h)\nu'^{s_h}_h(b_h)\\
    \leq &\sum_{s_h\in\S}\sum_{a_h=\mu_h(s_h),b_h=\nu'_h(s_h)}\frac{C^t(\mu,\nu)\iota^t}{n^t}\Sp{d_h^{\mu,\nu'}(s_h,a_h,b_h)}^{1-t}\\
    \leq&\Sp{SC(\mu,\nu)\iota/n}^{t}.\tag{Cauchy-Schwarz Inequality}
\end{align*}
\end{proof}

\begin{theorem}\label{thm:policy}
With probability $1-\delta$, we have
$$\mathrm{Gap}(\pi^{\mathrm{output}})\leq\min_{\pi=(\mu,\nu)\in\Pi}\Mp{\mathrm{Gap}(\pi)+4H^2\sqrt{S\log(\N(\Pi))C(\pi)\iota^2/n}+2HC(\pi)S(A+B)\iota/n}.$$
In additon, if $n\geq \frac{8\log(SABH/\delta)}{p_\mathrm{min}}$, we have
$$\mathrm{Gap}(\pi^{\mathrm{output}})\leq\min_{\pi=(\mu,\nu)\in\Pi}\Mp{\mathrm{Gap}(\pi)+8H^2\sqrt{S\log(\N(\Pi))C(\pi)\iota^2/n}}.$$
\end{theorem}

\begin{proof}
The first argument can be derived by Lemma \ref{lemma:bound on sumb policy} and Lemma \ref{lemma:bound on sump policy} with $t=1$. The second argument can be derived by Theorem \ref{thm:dataset} and Lemma \ref{lemma:warmup}. 
\end{proof}

\begin{corollary}\label{crl:full apx}
If $\Pi=\Pi^\mathrm{full}$, then with probability $1-\delta$ we have
$$\mathrm{Gap}(\pi^{\mathrm{output}})=\widetilde{O}(\sqrt{H^4S(A+B)C(\pi^*)/n}). $$
In addition, for turn-based two-player zero-sum Markov games, we can set $\Pi=\Pi^{\mathrm{det}}$ and we have
$$\mathrm{Gap}(\pi^{\mathrm{output}})=\widetilde{O}(\sqrt{H^4SC(\pi^*)/n}).$$
\end{corollary}

\begin{proof}
The first argument can be derived by Lemma \ref{lemma:covering number of Pi} and Theorem \ref{thm:policy} with $t=1/2$. The second argument can be derived by Lemma \ref{lemma:finite pi}, Lemma \ref{lemma:bound on sumb policy} and Lemma \ref{lemma:bound on sump policy} with $t=1/2$.
\end{proof}

\section{Proofs in Section \ref{sec:mpgs}}\label{apx:mpgs}

\begin{lemma}\label{lemma:best response}
For any strategy $\pi\in\Pi$, $h\in[H]$ and $s_h\in\S$, we have
$$\ov_{h,j}^{*,\pi_{-j}}(s_h)=\max_{\pi_j}\ov_{h,j}^{\pi}(s_h). $$
\end{lemma}

\begin{proof}
We prove this argument by induction. It holds trivially for $H+1$ as $\ov_{H+1,j}^{*,\pi_{-j}}(s)=\max_{\pi_j}\ov_{H+1,j}^{\pi}(s)=0$ for any $s\in\S$. Suppose the argument holds for $h+1$ and now we consider $h$. 

Consider function
\begin{align*}
    f(\pi'^s_{h,j})=&\E_{a_j\sim\pi'^s_{h,j},\a_{-j}\sim\pi^s_{h,-j}}\widehat{r}_{h,j}(s,a_j,\a_{-j})+\E_{a_j\sim\pi'^s_{h,j},\a_{-j}\sim\pi^s_{h,-j}}\widehat{P}_h(s,a_j,\a_{-j})\cdot \ov_{h+1,j}^{*,\pi_{-j}}\\
    &+b_h(s,\pi'^s_{h,j},\pi_{h,-j}^s)+H\sum_{\a_{-j}:(a_j,\a_{-j})\notin\K(s)}\pi_{h,-j}^s(\a_{-j}). 
\end{align*}
Lemma \ref{lemma:mpgs convex} shows that $b_h(s,\pi'^s_{h,j},\pi_{h,-j}^s)$ is convex with respect to $\pi'^s_{h,j}$, while all the other terms are linear with respect to $\pi'^s_{h,j}$. As a result, $f(\pi'^s_{h,j})$ is a convex function and thus we have
$$\max_{\pi'^s_{h,j}\in\Delta(\A_j)}f(\pi'^s_{h,j})=\max_{\pi'^s_{h,j}\in D(\A_j)}f(\pi'^s_{h,j}). $$
Then we have
\begin{align*}
    &\max_{a_j\in\A_j}\ov_{h,j}(s,a_j)\\
    =&\max_{\pi'^s_{h,j}\in D(\A_j)}f(\pi'^s_{h,j})\\
    =&\max_{\pi'^s_{h,j}\in\Delta(\A_j)}f(\pi'^s_{h,j})\\
    =&\max_{\pi'^s_{h,j}\in\Delta(\A_j)}\E_{a_j\sim\pi'^s_{h,j},\a_{-j}\sim\pi^s_{h,-j}}\widehat{r}_{h,j}(s,a_j,\a_{-j})+\E_{a_j\sim\pi'^s_{h,j},\a_{-j}\sim\pi^s_{h,-j}}\widehat{P}_h(s,a_j,\a_{-j})\cdot \ov_{h+1,j}^{*,\pi_{-j}}\\
    &+b_h(s,\pi'^s_{h,j},\pi_{h,-j}^s)+H\sum_{\a_{-j}:(a_j,\a_{-j})\notin\K(s)}\pi_{h,-j}^s(\a_{-j})\\
    =&\max_{\pi'^s_{h,j}\in\Delta(\A_j)}\E_{a_j\sim\pi'^s_{h,j},\a_{-j}\sim\pi^s_{h,-j}}\widehat{r}_{h,j}(s,a_j,\a_{-j})+\max_{\pi_j}\E_{a_j\sim\pi'^s_{h,j},\a_{-j}\sim\pi^s_{h,-j}}\widehat{P}_h(s,a_j,\a_{-j})\cdot \ov_{h+1,j}^{\pi}\\
    &+b_h(s,\pi'^s_{h,j},\pi_{h,-j}^s)+H\sum_{\a_{-j}:(a_j,\a_{-j})\notin\K(s)}\pi_{h,-j}^s(\a_{-j})\tag{Induction hypothesis}\\
    =&\max_{\pi_j}\E_{a_j\sim\pi'^s_{h,j},\a_{-j}\sim\pi^s_{h,-j}}\widehat{r}_{h,j}(s,a_j,\a_{-j})+\E_{a_j\sim\pi'^s_{h,j},\a_{-j}\sim\pi^s_{h,-j}}\widehat{P}_h(s,a_j,\a_{-j})\cdot \ov_{h+1,j}^{\pi}\\
    &+b_h(s,\pi'^s_{h,j},\pi_{h,-j}^s)+H\sum_{\a_{-j}:(a_j,\a_{-j})\notin\K(s)}\pi_{h,-j}^s(\a_{-j}). \\
\end{align*}
So we have $\ov_{h,j}^{*,\pi_{-j}}(s_h)=\max_{\pi_j}\ov_{h,j}^{\pi}(s_h)$. (See Algorithm \ref{alg:value estimation} and Algorithm \ref{alg:best response estimation} for the definition of both quantities)
\end{proof}

\begin{lemma}\label{lemma:marl concentration single}
Fix $\pi'\in\Pi, j\in[m], h\in[H]$ and $s\in\S$, with probability $1-\delta$ we have
\begin{align*}
    \abs{\sum_{\a\in\K_h(s)}\pi_{h}'(\a|s)\Sp{r_{h,j}(s,\a)+\inner{P_h(s,\a),\uv^{\pi'}_{h+1,j}}-\widehat{r}_{h,j}(s,\a)-\inner{\widehat{P}_{h,j}(s,\a),\uv^{\pi'}_{h+1,j}}}}\\
    \leq H\sqrt{2\sum_{\a\in\K_h(s)}\frac{\pi'_h(\a|s)^2}{n_h(s,\a)}\log(4/\delta)},
\end{align*}
and
\begin{align*}
    \abs{\sum_{\a\in\K_h(s)}\pi_{h}'(\a|s)\Sp{r_{h,j}(s,\a)+\inner{P_h(s,\a),\ov^{\pi'}_{h+1,j}}-\widehat{r}_{h,j}(s,\a)-\inner{\widehat{P}_h(s,\a),\ov^{\pi'}_{h+1,j}}}}\\
    \leq H\sqrt{2\sum_{\a\in\K_h(s)}\frac{\pi'_h(\a|s)^2}{n_h(s,\a)}\log(4/\delta)}.
\end{align*}

\end{lemma}

\begin{proof}
We use $k_h^i(s,a,b)$ to denote the index of $(s,a,b)$ appears in the dataset at timestep $h$ for $i$th time. With probability $1-\delta$, we have
\begin{align*}
    &\abs{\sum_{(\a)\in\K_h(s)}\pi_{h}'(\a|s)\Sp{r_{h,j}(s,\a)+\inner{P_h(s,\a),\uv^{\pi'}_{h+1,j}}-\widehat{r}_{h,j}(s,\a)-\inner{\widehat{P}_h(s,\a),\uv^{\pi'}_{h+1,j}}}}\\
    =&\left|\sum_{\a\in\K_h(s)}\sum_{i=1}^{n_h(s,\a)}\frac{\pi_{h}'(\a|s)}{n_h(s,\a)}\Sp{r_{h,j}^{k_h^i(s,\a)}-r_{h,j}(s,\a)}\right.\\&+\left.\sum_{(\a)\in\K_h(s)}\sum_{i=1}^{n_h(s,\a)}\frac{\pi_{h}'(\a|s)}{n_h(s,\a)}\Sp{\uv^{\pi'}_{h+1,j}(s_{h+1}^{k_h^i(s,\a)})-\inner{P_h(s,\a),\uv^{\pi'}_{h+1,j}}}\right|\\
    \leq& \sqrt{\frac{1}{2}\sum_{\a\in\K_h(s)}\frac{\pi_{h}'(\a|s)^2}{n_h(s,\a)}\log(2/\delta)}+H\sqrt{\frac{1}{2}\sum_{\a\in\K_h(s)}\frac{\pi_{h}'(\a|s)^2}{n_h(s,\a)}\log(2/\delta)}\\
    \leq& H\sqrt{2\sum_{\a\in\K_h(s)}\frac{\pi_{h}'(\a|s)^2}{n_h(s,\a)}\log(2/\delta)},
\end{align*}
where the first inequality is from Hoeffding's inequality and the fact that $\uv_{h+1,j}$ has no dependence on the dataset at timestep $h$. The second argument holds similarly. Rescaling $\delta$ to $\delta/2$ and with an union bound we can prove the lemma. 
\end{proof}

\begin{lemma}\label{lemma:marl concentration all}
With probability $1-\delta$, for all $\pi\in\Pi,j\in[m],h\in[H],s\in\S$, we have
$$\abs{\sum_{\a\in\K_h(s)}\pi_{h}(\a|s)\Sp{r_{h,j}(s,\a)+\inner{P_h(s,\a),\uv^\pi_{h+1,j}}-\widehat{r}_{h,j}(s,\a)-\inner{\widehat{P}_h(s,\a),\uv^\pi_{h+1,j}}}}\leq b_h(s,\pi_h^s),$$
$$\abs{\sum_{\a\in\K_h(s)}\pi_{h}(\a|s)\Sp{r_{h,j}(s,\a)+\inner{P_h(s,\a),\ov^\pi_{h+1,j}}-\widehat{r}_{h,j}(s,\a)-\inner{\widehat{P}_h(s,\a),\ov^\pi_{h+1,j}}}}\leq b_h(s,\pi_h^s).$$
Denote this event as $\G_{\mathrm{marl}}$.
\end{lemma}

\begin{proof}

We prove the argument for $\uv_{h+1,j}^\pi$ and the argument for $\ov_{h+1,j}^\pi$ holds similarly. Suppose $\mathcal{V}$ is a $\epsilon_\mathrm{cover}$-covering of $[0,H]^S$ with respect to L-$\infty$ norm and $|\mathcal{V}|\leq(1+HS/\epsilon_\mathrm{cover})^S$.  First, using a union bound for all $j\in[m],h\in[H],s\in\S,\pi'^s_{h,j}\in\C(\Pi^{\mathrm{prior}}_{h,j}(s)),V_{h+1}\in \mathcal{V}$ on Lemma \ref{lemma:marl concentration single}, with probability $1-\delta$ we have
\begin{align*}
    &\abs{\sum_{\a\in\K_h(s)}\pi'_h(\a|s)\Sp{r_{h,j}(s,\a)+\inner{P_h(s,\a),V_{h+1}}-\widehat{r}_{h,j}(s,\a)-\inner{\widehat{P}_h(s,\a),V_{h+1}}}}\\
    \leq& H\sqrt{4\sum_{\a\in\K_h(s)}\frac{\pi'_h(\a|s)^2}{n_h(s,\a)}\log(4m\sum_{s\in\S,h\in[H]}\prod_{j\in[m]}|\C(\Pi_{h,j}(s))|(1+HS/\epsilon_\mathrm{cover})^S/\delta)}\\
    \leq& H\sqrt{8\sum_{\a\in\K_h(s)}\frac{\pi'_h(\a|s)^2}{n_h(s,\a)}S\log(8m\N(\Pi)SH/\epsilon_\mathrm{cover}\delta)}.
\end{align*}

Note that $r_{h,j}(s,\a)+\inner{P_h(s,\a),\uv^\pi_{h+1,j}}-\widehat{r}_{h,j}(s,\a)-\inner{\widehat{P}_h(s,\a),\uv^\pi_{h+1,j}}$ is bounded in $[-H,H]$ as $r_{h,j}(s,\a)\in[0,1]$ and $\uv^\pi_{h+1,j}\in[0,H-h]$. There exists $V_{h+1}\in\mathcal{V}$ such that $\|\uv^\pi_{h+1,j}-V_{h+1}\|_\infty\leq\epsilon_\mathrm{cover}$, which implies
\begin{align*}
    &\abs{\sum_{\a\in\K_h(s)}\pi'_h(\a|s)\Sp{r_h(s,\a)+\inner{P_h(s,\a),V_{h+1}}-\widehat{r}_h(s,\a)-\inner{\widehat{P}_h(s,\a),V_{h+1}}}}\\
    -&\abs{\sum_{\a\in\K_h(s)}\pi'_h(\a|s)\Sp{r_h(s,\a)+\inner{P_h(s,\a),\uv^\pi_{h+1,j}}-\widehat{r}_h(s,\a)-\inner{\widehat{P}_h(s,\a),\uv^\pi_{h+1,j}}}}\\
    \leq&2\epsilon_\mathrm{cover}. 
\end{align*}

For any $\pi^s_{h,j}\in\Pi_{h,j}(s)$, there exists $\pi'^s_{h,j}\in\C(\Pi_{h,j}(s))$ such that $\|\pi_{h,j}(\cdot|s)-\pi'_{h,j}(\cdot|s)\|_1\leq\epsilon_\mathrm{cover}$ for all $j\in[m]$ and $s\in\S$. So with Lemma \ref{lemma:net error}, we have
\begin{align*}
    &\left|\sum_{\a\in\K_h(s)}\pi'_h(\a|s)\Sp{r_{h,j}(s,\a)+\inner{P_h(s,\a),\uv^\pi_{h+1,j}}-\widehat{r}_{h,j}(s,\a)-\inner{\widehat{P}_h(s,\a),\uv^\pi_{h+1,j}}}\right.\\
    &-\left.\sum_{\a\in\K_h(s)}\pi_h(\a|s)\Sp{r_h(s,\a)+\inner{P_h(s,\a),\uv^\pi_{h+1,j}}-\widehat{r}_{h,j}(s,\a)-\inner{\widehat{P}_h(s,\a),\uv^\pi_{h+1,j}}}\right|\\
    \leq& m\epsilon_\mathrm{cover} H.
\end{align*}

By Lemma \ref{lemma:net error 2}, we have
\begin{align*}
    \abs{\sqrt{\sum_{\a\in\K_h(s)}\frac{\pi'_h(\a|s)^2}{n_h(s,\a)}}-\sqrt{\sum_{\a\in\K_h(s)}\frac{\pi_h(\a|s)^2}{n_h(s,\a)}}}\leq \sqrt{2m\epsilon_\mathrm{cover}}.
\end{align*}

Combining all these parts together and then with probability $1-\delta$, we have
\begin{align*}
    &\abs{\sum_{\a\in\K_h(s)}\pi_h(\a|s)\Sp{r_{h,j}(s,\a)+\inner{P_h(s,\a),\uv^\pi_{h+1,j}}-\widehat{r}_{h,j}(s,\a)-\inner{\widehat{P}_h(s,\a),\uv^\pi_{h+1,j}}}}\\
    \leq& H\sqrt{8\sum_{\a\in\K_h(s)}\frac{\pi_h(\a|s)^2}{n_h(s,\a)}S\log(8m\N(\Pi,\epsilon_{\mathrm{cover}})SH\delta)}+2\epsilon_\mathrm{cover}+m\epsilon_\mathrm{cover} H\\
    &+H\sqrt{8m\epsilon_\mathrm{cover} \log(8m\N(\Pi,\epsilon_\mathrm{cover})SH/\delta)}.
\end{align*}
By Lemma \ref{L-1 covering number}, we have
\begin{align*}
    \N(\Pi,\epsilon_{\mathrm{cover}})=&\frac{1}{SH}\sum_{s\in\S,h\in[H]}\prod_{j\in[m]}|\C (\Pi_{h,j}(s),\epsilon_{\mathrm{cover}})|\\
    \leq&\prod_{j\in[m]}(3A_j/\epsilon_{\mathrm{cover}})^{A_j}\\
    \leq&(3(\sum_{j\in[m]}A_j)/\epsilon_{\mathrm{cover}})^{\sum_{j\in[m]}A_j}. 
\end{align*}
Set $\epsilon_\mathrm{cover}=\frac{1}{\sum_{j\in[m]}A_j mH^2n^2}$ and with some calculations we can get
\begin{align*}
    &\abs{\sum_{\a\in\K_h(s)}\pi_h(\a|s)\Sp{r_h(s,\a)+\inner{P_h(s,\a),\uv^\pi_{h+1,j}}-\widehat{r}_{h,j}(s,\a)-\inner{\widehat{P}_h(s,\a),\uv^\pi_{h+1,j}}}}\\
    \leq& H\sqrt{8\sum_{\a\in\K_h(s)}\frac{\pi_h(\a|s)^2}{n_h(s,\a)}S\log(8m\N(\Pi)SHn/\delta)}+\sqrt{32\log(16\prod_{j\in[m]}A_jmSHn/\delta)}/n\\
    \leq& H\sqrt{\sum_{\a\in\K_h(s)}\frac{\pi_h(a|s)^2}{n_h(s,\a)}S\log(\N(\Pi))\iota}+\sqrt{\iota}/n. 
\end{align*}
\end{proof}

\begin{lemma}\label{lemma:marl pessimism}
Under event $\G_{\mathrm{marl}}$, for all $j\in[m]$, $h\in[H]$, $\pi\in\Pi$ and $s\in\S$, we have
$$\uv_{h,j}^\pi(s)\leq V_{h,j}^\pi(s)\leq\ov_{h,j}^\pi(s).$$
\end{lemma}

\begin{proof}
We prove this argument by induction. It holds for $h=H+1$ as $\uv_{H+1,j}^\pi(s)=V_{H+1,j}^\pi(s)=\ov_{H+1,j}^\pi(s)$. Suppose the argument holds for $h+1$ and we consider $h$. 
\begin{align*}
    \uv_{h,j}^\pi(s)=&\proj_{[0,H-h+1]}\Bp{\E_{\a\sim\pi_h(\cdot|s)}\widehat{r}_{h,j}(s,\a)+\E_{\a\sim\pi_h(\cdot|s)}\widehat{P}_h(s,\a)\cdot \uv_{h+1,j}^\pi-b_h(s,\pi^s_h)}\\
    =&\proj_{[0,H-h+1]}\Bp{\sum_{\a\in\K_h(s)}\pi_h(\a|s)\Sp{\widehat{r}_{h,j}(s,\a)+\inner{\widehat{P}_h(s,\a),\uv^\pi_{h+1,j}}}-b_h(s,\pi^s_h)}\\
    \leq&\proj_{[0,H-h+1]}\Bp{\sum_{\a\in\K_h(s)}\pi_h(\a|s)\Sp{r_{h,j}(s,\a)+\inner{P_h(s,\a),\uv^\pi_{h+1,j}}}}\tag{Lemma \ref{lemma:marl concentration all}}\\
    \leq&\proj_{[0,H-h+1]}\Bp{\sum_{\a\in\K_h(s)}\pi_h(\a|s)\Sp{r_{h,j}(s,\a)+\inner{P_h(s,\a),V^\pi_{h+1,j}}}}\tag{Induction hypothesis}\\
    \leq&\proj_{[0,H-h+1]}\Bp{\sum_{\a\in \A}\pi_h(\a|s)\Sp{r_{h,j}(s,\a)+\inner{P_h(s,\a),V^\pi_{h+1,j}}}}\\
    \leq& \proj_{[0,H-h+1]}\Bp{V_{h,j}^\pi(s)}\\
    =& V_{h,j}^\pi(s). 
\end{align*}
\begin{align*}
    &\ov_{h,j}^\pi(s)\\
    =&\proj_{[0,H-h+1]}\Bp{\E_{\a\sim\pi_h(\cdot|s)}\widehat{r}_{h,j}(s,\a)+\E_{\a\sim\pi_h(\cdot|s)}\widehat{P}_h(s,\a)\cdot \ov_{h+1,j}^\pi+b_h(s,\pi^s_h)+H\sum_{a\notin\K(s)}\pi_h(\a|s)}\\
    =&\proj_{[0,H-h+1]}\Bp{\sum_{\a\in\K_h(s)}\pi_h(\a|s)\Sp{\widehat{r}_{h,j}(s,\a)+\inner{\widehat{P}_h(s,\a),\ov^\pi_{h+1,j}}}+b_h(s,\pi^s_h)+H\sum_{a\notin\K(s)}\pi_h(\a|s)}\\
    \geq& \proj_{[0,H-h+1]}\Bp{\sum_{\a\in\K_h(s)}\pi_h(\a|s)\Sp{r_{h,j}(s,\a)+\inner{P_h(s,\a),\ov^\pi_{h+1,j}}}+H\sum_{a\notin\K(s)}\pi_h(\a|s)}\tag{Lemma \ref{lemma:marl concentration all}}\\
    \geq& \proj_{[0,H-h+1]}\Bp{\sum_{\a\in\A}\pi_h(\a|s)\Sp{r_{h,j}(s,\a)+\inner{P_h(s,\a),\ov^\pi_{h+1,j}}}}\tag{$\ov_{h+1,j}^\pi(s)\leq H-h$ for all $s\in\S$}\\
    \geq& \proj_{[0,H-h+1]}\Bp{\sum_{\a\in\A}\pi_h(\a|s)\Sp{r_{h,j}(s,\a)+\inner{P_h(s,\a),V^\pi_{h+1,j}}}}\tag{Induction hypothesis}\\
    =& \proj_{[0,H-h+1]}\Bp{V_{h,j}^\pi(s)}\\
    =& V_{h,j}^\pi(s). 
\end{align*}
\end{proof}

\begin{lemma}\label{lemma:marl gap bound}
Under event $\G_{\mathrm{marl}}$, for any policy $\pi\in\Pi$, we have
$$\mathrm{Gap}(\pi)\leq\sum_{j\in[m]}\ov_{1,j}^{*,\pi_{-j}}(s)-\uv_{1,j}^\pi(s). $$
In addition, we have
$$\mathrm{Gap}(\pi^{\mathrm{output}})\leq \min_{\pi\in\Pi}\sum_{j\in[m]}\Mp{\ov_{1,j}^{*,\pi_{-j}}(s)-\uv_{1,j}^{\pi}(s)}.$$
\end{lemma}

\begin{proof}
By Lemma \ref{lemma:marl pessimism}, we have
$$\mathrm{Gap}(\pi)=\max_{\pi'}\sum_{j\in[m]}V_{1,j}^{\pi'_j,\pi_{-j}}(s)-V_{1,j}^\pi(s)\leq\max_{\pi'}\sum_{j\in[m]}\ov_{1,j}^{\pi'_j,\pi_{-j}}(s)-\uv_{1,j}^\pi(s).$$
Combined with Lemma \ref{lemma:best response} we can prove the first argument. For the second argument,
note that $\pi_\mathrm{output}$ is the minimizer of the RHS, so we have
$$\mathrm{Gap}(\pi^{\mathrm{output}})\leq\min_{\pi\in\Pi}\sum_{j\in[m]}\ov_{1,j}^{*,\pi_{-j}}(s)-\uv_{1,j}^{\pi}(s).$$
\end{proof}

\begin{lemma}\label{lemma:strategy error}
Under event $\G_{\mathrm{marl}}$, for any strategy $\pi\in\Pi$, we have
$$\uv_{1,j}^\pi(s_1)\geq V_{1,j}^{\pi}(s_1)-\E_\pi\sum_{h\in[H]}\widehat{b}_h(s_h,\pi^{s_h}_h),\ \ov_{1,j}^\pi(s_1)\leq V_{1,j}^{\pi}(s_1)+\E_\pi\sum_{h\in[H]}\widehat{b}_h(s_h,\pi^{s_h}_h). $$
\end{lemma}

\begin{proof}
We prove the first argument and the second argument holds similarly. 
\begin{align*}
    &V_{1,j}^\pi(s_1)-\uv_{1,j}^\pi(s_1)\\
    =&\E_{\a\sim\pi_1(\cdot|s_1)}\Mp{r_{1,j}(s_1,\a)+P_1(s_1,\a)\cdot V_{2,j}^\pi}-\E_{\a\sim\pi_1(\cdot|s_1)}\Mp{\widehat{r}_{1,j}(s_1,\a)+\widehat{P}_1(s_1,\a)\cdot \uv_{2,j}^\pi}+b_1(s_1,\pi^{s_1}_1)\\
    =&\E_{\pi_1}\Mp{V_{2,j}^\pi(s_2)-\uv_{2,j}^\pi(s_2)}+\E_{\pi_1}\Mp{r_{1,j}(s_1,\a)+P_1(s_1,\a)\cdot \uv_{2,j}^\pi-\widehat{r}_{1,j}(s_1,\a)-\widehat{P}_1(s_1,\a)\cdot \uv_{2,j}^\pi}+b_1(s_1,\pi^{s_1}_1)\\
    \leq&\E_{\pi_1}\Mp{V_{2,j}^\pi(s_2)-\uv_{2,j}^\pi(s_2)}+\sum_{\a\in\K_h(s_1)}\pi_1(\a|s_1)\Sp{r_{1,j}(s_1,\a)+P_1(s_1,\a)\cdot V_{2,j}^\pi-\widehat{r}_{1,j}(s_1,\a)-\widehat{P}_1(s_1,\a)\cdot V_{2,j}^\pi}\\&+\sum_{\a\notin\K_h(s_1)}\pi(\a|s_1)H+b_1(s_1,\pi^{s_1}_1)\\
    \leq&\E_{\pi_1}\Mp{V_{2,j}^\pi(s_2)-\uv_{2,j}^\pi(s_2)}+\sum_{\a\notin\K_h(s_1)}\pi_1(\a|s_1)H+2b_1(s_1,\pi^{s_1}_1)\\
    =&\E_{\pi_1}\Mp{V_{2,j}^\pi(s_2)-\uv_{2,j}^\pi(s_2)}+\widehat{b}_1(s_1,\pi^{s_1}_1). 
\end{align*}
By telescoping we can prove the first argument. 
\end{proof}

\begin{lemma}\label{lemma:marl gap decomp}
Under good event $\G_{\mathrm{marl}}$, for any strategy $\pi\in\Pi$, we have
\begin{align*}
    \sum_{j\in[m]}\ov_{1,j}^{*,\pi_{-j}}(s_1)-\uv_{1,j}^{\pi}(s_1)\leq&\mathrm{Gap}(\pi)+ \max_{\pi'\in\Pi^{\mathrm{det}}}\sum_{j\in[m]}\E_{\pi'_j,\pi_{-j}}\Mp{\sum_{h=1}^H \widehat{b}_h(s_h,\pi'^{s_h}_{h,j},\pi^{s_h}_{h,-j})}+m\E_{\pi}\sum_{h=1}^H \Mp{\widehat{b}_h(s_h,\pi^{s_h}_h)}. 
\end{align*}
\end{lemma}

\begin{proof}
Set $\widetilde{\pi}=\argmax_{\pi'\in\Pi^{\mathrm{full}}}\sum_{j\in[m]}\ov_{1,j}^{\pi'_j,\pi_{-j}}(s_1)-\uv_{1,j}^{\pi}(s_1)$. Lemma \ref{lemma:best response} shows that there always exists a deterministic strategy $\widetilde{\pi}\in\Pi^{\mathrm{det}}$, which is used by Algorithm \ref{alg:best response estimation}. 

\begin{align*}
    &\max_{\pi'\in\Pi^{\mathrm{full}}}\sum_{j\in[m]}\ov_{1,j}^{\pi'_j,\pi_{-j}}(s_1)-\uv_{1,j}^{\pi}(s_1)\\
    =&\sum_{j\in[m]}\ov_{1,j}^{\widetilde{\pi}_j,\pi_{-j}}(s_1)-\uv_{1,j}^{\pi}(s_1)\\
    \leq& \sum_{j\in[m]}\Mp{V_{1,j}^{\widetilde{\pi}_j,\pi_{-j}}(s_1)-V_{1,j}^{\pi}(s_1)+\E_{\widetilde{\pi}_j,\pi_{-j}}\sum_{h\in[H]}\widehat{b}_h(s_h,\widetilde{\pi}^{s_h}_{h,j},\pi^{s_h}_{h,-j})+\E_\pi\sum_{h\in[H]}\widehat{b}_h(s_h,\pi^{s_h}_h)}\tag{Lemma \ref{lemma:strategy error}}\\
    \leq& \max_{\pi'\in\Pi^{\mathrm{det}}}\sum_{j\in[m]}\Mp{V_{1,j}^{\pi'_j,\pi_{-j}}(s_1)-V_{1,j}^{\pi}(s_1)}+\sum_{j\in[m]}\E_{\widetilde{\pi}_j,\pi_{-j}}\Mp{\sum_{h=1}^H \widehat{b}_h(s_h,\widetilde{\pi}^{s_h}_{h,j},\pi^{s_h}_{h,-j})}+m\E_{\pi}\sum_{h=1}^H \Mp{\widehat{b}_h(s_h,\pi^{s_h}_h)}\\
    \leq&\mathrm{Gap}(\pi)+ \max_{\pi'\in\Pi^{\mathrm{det}}}\sum_{j\in[m]}\E_{\pi'_j,\pi_{-j}}\Mp{\sum_{h=1}^H \widehat{b}_h(s_h,\pi'^{s_h}_{h,j},\pi^{s_h}_{h,-j})}+m\E_{\pi}\sum_{h=1}^H \Mp{\widehat{b}_h(s_h,\pi^{s_h}_h)}. 
\end{align*}

\end{proof}

\subsection{Dataset-dependent Bound}

\begin{lemma}\label{lemma:marl bound on sumb dataset}
Suppose $\widehat{C}(\pi)$ is finite. For any strategy $\pi'\in\Pi$, $h\in[H]$ and $j\in[m]$, we have
$$\E_{\pi'_j,\pi_{-j}}b_h(s_h,\pi'^{s_h}_{h,j},\pi^{s_h}_{h,-j})\leq 2HS\sqrt{\widehat{C}(\pi)\log(\N(\Pi))\iota/n}.$$
\end{lemma}

\begin{proof}
\begin{align*}
    &\E_{\pi'_j,\pi_{-j}}b_h(s_h,\pi'^{s_h}_{h,j},\pi^{s_h}_{h,-j})\\
    =&\E_{\pi'_j,\pi_{-j}}H\sqrt{\sum_{\a\in\K_h(s_h)}\frac{(\pi'_{h,j},\pi_{h,-j})(\a|s_h)^2}{n_h(s,\a)}S\log(\N(\Pi))\iota}+\sqrt{\iota}/n\\
    =&\sum_{s_h\in\S}H\sqrt{\sum_{\a\in\K_h(s_h)}\frac{d_h^{\pi'_j,\pi_{-j}}(s_h)(\pi'_{h,j},\pi_{h,-j})(\a|s_h)^2}{n_h(s_h,\a)}S\log(\N(\Pi))\iota}+\sqrt{\iota}/n\\
    =&\sum_{s_h\in\S}H\sqrt{\sum_{\a\in\K_h(s_h)}\frac{d_h^{\pi'_j,\pi_{-j}}(s_h,\a)^2}{n\widehat{d}_h(s_h,\a)}S\log(\N(\Pi))\iota}+\sqrt{\iota}/n\\
    \leq& \sum_{s_h\in\S}H\sqrt{\sum_{\a\in\K_h(s_h)}\widehat{C}(\pi)d_h^{\pi'_j,\pi_{-j}}(s_h,\a)S\log(\N(\Pi))\iota/n}+\sqrt{\iota}/n\\
    \leq& H\sqrt{S^2\widehat{C}(\pi)\log(\N(\Pi))\iota/n}+\sqrt{\iota}/n \tag{Cauchy-Schwarz inequality}\\
    \leq& 2HS\sqrt{\widehat{C}(\pi)\log(\N(\Pi))\iota/n}. 
\end{align*}
\end{proof}

\begin{lemma}\label{lemma:marl bound on sump dataset}
Suppose $\widehat{C}(\pi)$ is finite. For any strategy $\pi'\in\Pi$, $h\in[H]$ and $j\in[m]$, we have
$$\E_{\pi'_j,\pi_{-j}}\sum_{\a_h\notin\K_h(s_h)}(\pi'_{h,j},\pi_{h,-j})(\a_h|s_h)=0. $$
\end{lemma}

\begin{proof}
Similar to Lemma \ref{lemma:bound on sump dataset}, we have
\begin{align*}
    &\E_{\pi'_j,\pi_{-j}}\sum_{\a_h\notin\K_h(s_h)}(\pi'_{h,j},\pi_{h,-j})(\a_h|s_h)\\
    =&\E_{\pi'_j,\pi_{-j}}\sum_{\a_h:\widehat{d}_h(s_h,\a_h)=0}(\pi'_{h,j},\pi_{h,-j})(\a_h|s_h)\\
    =&\sum_{\a:\widehat{d}_h(s_h,\a_h)=0}d_h^{\pi'_j,\pi_{-j}}(s_h,\a_h)\\
    \leq& \widehat{C}(\pi)\sum_{\a:\widehat{d}_h(s_h,\a_h)=0}\widehat{d}_h(s_h,\a_h)\\
    =&0. 
\end{align*}
\end{proof}

\begin{lemma}\label{lemma:marl bhat bound}
For any strategy $\pi\in\Pi$ and $j\in[m]$, we have
$$\max_{\pi'}\E_{\pi'_j,\pi_{-j}}\Mp{\sum_{h=1}^H\widehat{b}_h(s_h,\pi'^{s_h}_{h,j},\pi^{s_h}_{h,-j})}\leq2H^2S\sqrt{\widehat{C}(\pi)\log(\N(\Pi))\iota/n}. $$
\end{lemma}

\begin{proof}
If $\widehat{C}(\pi)$ is infinite, the argument holds directly. Otherwise it can be derived from Lemma \ref{lemma:marl bound on sumb dataset} and Lemma \ref{lemma:marl bound on sump dataset}. 
\end{proof}

\begin{theorem}\label{thm:marl dataset}
With probability $1-\delta$, we have
$$\mathrm{Gap}(\pi^{\mathrm{output}})\leq\min_{\pi\in\Pi}\Mp{\mathrm{Gap}(\pi)+4mH^2S\sqrt{\widehat{C}(\pi)\log(\N(\Pi))\iota/n}}.$$
\end{theorem}

\begin{proof}
This can be derived from Lemma \ref{lemma:marl bhat bound}, Lemma \ref{lemma:marl gap bound} and Lemma \ref{lemma:marl gap decomp}. 
\end{proof}

\subsection{Dataset-independent Bound}

\begin{lemma}\label{lemma:marl warmup}
Suppose $p_\mathrm{min}=\min_{s,\a,h}\{d_h^\rho(s,\a):d_h^\rho(s,\a)>0\}$. With probability $1-\delta$, for all $h,s,\a$, we have
$$n_h(s,\a)\geq \Sp{1-\sqrt{\frac{2\log(S\Pi_{j\in[m]}A_jH/\delta)}{np_\mathrm{min}}}}nd_h(s,\a). $$
As a result, if $n\geq \frac{8\log(S\Pi_{j\in[m]}A_jH/\delta)}{p_\mathrm{min}}$, for all strategy $\pi$, we have
$$2C(\pi)\geq \widehat{C}(\pi). $$
\end{lemma}

\begin{proof}
For a fixed $s,\a,h$, for any $\epsilon>0$ we have
$$\P(n_h(s,\a)<(1-\epsilon) nd_h(s,\a))\leq\exp\Sp{-\frac{\epsilon^2nd_h(s,\a)}{2}}\leq\exp\Sp{-\frac{\epsilon^2np_\mathrm{min}}{2}}.$$
With a union bound, we have
$$\P(\exists h,s,a,b: \P(n_h(s,a,b)<(1-\epsilon) nd_h(s,a,b)))\leq S\Pi_{j\in[m]}A_jH\exp\Sp{-\frac{\epsilon^2np_\mathrm{min}}{2}}.$$
The RHS is smaller than $\delta$ if we set
$$\epsilon=\sqrt{\frac{2\log(S\Pi_{j\in[m]}A_jH/\delta)}{np_\mathrm{min}}}. $$
If $n\geq \frac{8\log(S\Pi_{j\in[m]}A_jH/\delta)}{p_\mathrm{min}}$, we have
$\widehat{d}_h(s,\a)=\frac{n_h(s,\a)}{n}\geq\frac{d_h(s,\a)}{2}. $
By Definition \ref{def:marl dataset} and Definition \ref{def:policy}, we have
$$2C(\pi)\geq \widehat{C}(\pi).$$
\end{proof}

\begin{theorem}\label{thm:marl policy}
If $n\geq \frac{8\log(S\Pi_{j\in[m]}A_jH/\delta)}{p_\mathrm{min}}$, with probability $1-\delta$, we have
$$\mathrm{Gap}(\pi^{\mathrm{output}})\leq\min_{\pi\in\Pi}\Mp{\mathrm{Gap}(\pi)+4mH^2S\sqrt{2C(\pi)\log(\mathcal{N}(\Pi)\iota/n}}.$$
\end{theorem}

\begin{proof}
This can be derived by Lemma \ref{lemma:marl warmup} and Theorem \ref{thm:marl dataset}. 
\end{proof}

\section{Technical Lemmas}

\begin{lemma}\label{L-1 covering number}
(L-1 covering number of probability simplex) For probability simplex $\Delta(\A)$ and $A=|\A|$, there exists a subset $\Delta'(\A)\subset \Delta(\A)$ such that for any $p\in\Delta(\A)$, there exists $p'\in(\A)$ such that $\Norm{p-p'}_1\leq\epsilon$. In addition,
$$|\Delta'(\A)|\leq\Sp{\frac{3A}{\epsilon}}^{A}.$$
\end{lemma}

\begin{proof}
We construct $\epsilon'$-net for $\epsilon/2<\epsilon'\leq \epsilon$ such that $1/\epsilon'$ is integer. Then this $\epsilon'$-net is directly a $\epsilon$-net as $\epsilon'\leq\epsilon$. Define $D(\A)=\{(n_1\epsilon',n_2\epsilon',\cdots,n_A\epsilon'),\sum_{i=1}^A=\frac{1}{\epsilon'},n_i\in[0,1/\epsilon']\}\subset\Delta(\A)$. For $p=(p_1,p_2,\cdots,p_A)\in\Delta(\A)$, suppose 
$$k_i\epsilon'\leq p_i<(k_i+1)\epsilon',$$
for some non-negative integers $\{k_i\}$. Set $k=\sum_{i=1}^A k_i$ Then we have $1/\epsilon'-A<k\leq 1/\epsilon'$. Now we construct $p'=(n_1\epsilon',n_2\epsilon',\cdots,n_A\epsilon')\in D(\A)$ such that
$$\begin{cases}
n_i=k_i+1,&i\in[1/\epsilon'-k]\\
n_i=k_i,&\mathrm{otherwise.}
\end{cases}$$
Then we have $|p_i-p'_i|\leq \epsilon'$ for all $i\in[A]$, which implies
$$\|p-p'\|\leq A\epsilon'.$$
So $|D(\A)|\leq\Sp{\frac{1+\epsilon'}{\epsilon'}}^A\leq\Sp{\frac{3}{\epsilon}}^A$ is an $A\epsilon$-net of $\Delta(\A)$. We can prove the lemma by rescaling $\epsilon$.
\end{proof}

\begin{lemma}\label{lemma:net error}
Suppose $\pi_j,\pi'_j\in\Delta(\A_j)$  such that $\Norm{\pi_j-\pi'_j}_1\leq\epsilon$ for all $j\in[m]$. For any function $f(\a)\in[-H,H]$, we have
$$\abs{\E_{\a\sim\pi}f(\a)-\E_{\a\sim\pi'}f(\a)}\leq m\epsilon H.$$
\end{lemma}

\begin{proof}
\begin{align*}
    &\abs{\E_{\a\sim\pi}f(\a)-\E_{\a\sim\pi'}f(\a)}\\
    =&\abs{\sum_{\a}\Pi_{j=1}^m\pi_j(a_j)f(\a)-\sum_{\a}\Pi_{j=1}^m\pi'_j(a_j)f(\a)}\\
    =&\abs{\sum_{j=1}^m\sum_{\a_{-j}\in\Pi_{i\neq j}\A_i}\Pi_{i=1}^{j-1}\pi_i(a_i)\Pi_{i=j+1}^m\pi'_i(a_i)\sum_{a_j\in\A_j}\Sp{\pi_j(a_j)-\pi'_j(a_j)}f(\a)}\\
   \leq& \abs{\sum_{j=1}^m\sum_{\a_{-j}\in\Pi_{i\neq j}\A_i}\Pi_{i=1}^{j-1}\pi_i(a_i)\Pi_{i=j+1}^m\pi'_i(a_i)\epsilon H}\\
    =&m\epsilon H.
\end{align*}
\end{proof}

\begin{lemma}\label{lemma:net error 2}
Suppose $\pi_j,\pi'_j\in\Delta(\A_j)$  such that $\Norm{\pi_j-\pi'_j}_1\leq\epsilon$ for all $j\in[m]$. For any set $\K\subset\Pi_{j\in[m]}\A_j$ and function $n(\a)\geq1$ we have
$$\abs{\sqrt{\sum_{\a\in\K}\frac{\pi(\a)^2}{n(\a)}}-\sqrt{\sum_{\a\in\K}\frac{\pi'(\a)^2}{n(\a)}}}\leq \sqrt{2m\epsilon}.$$
\end{lemma}

\begin{proof}
\begin{align*}
    &\abs{\sqrt{\sum_{\a\in\K}\frac{\pi(\a)^2}{n(\a)}}-\sqrt{\sum_{\a\in\K}\frac{\pi'(\a)^2}{n(\a)}}}\\
    \leq&\sqrt{\abs{\sum_{\a\in\K}\frac{\pi(\a)^2-\pi'(\a)^2}{n(\a)}}}\\
    =&\sqrt{\abs{\sum_{j=1}^m\sum_{\a_{-j}\in\prod_{i\neq j}\A_i}\prod_{i=1}^{j-1}\pi^2_i(a_i)\prod_{i=j+1}^m\pi'^2_i(a_i)\sum_{a_j\in\A_j}\Sp{\pi^2_j(a_j)-\pi'^2_j(a_j)}\1(\a\in\K)/n(\a)}}\\
    \leq&\sqrt{\abs{\sum_{j=1}^m\sum_{\a_{-j}\in\prod_{i\neq j}\A_i}\prod_{i=1}^{j-1}\pi^2_i(a_i)\prod_{i=j+1}^m\pi'^2_i(a_i)2\epsilon}}\\
    \leq&\sqrt{2m\epsilon}.
\end{align*}
\end{proof}

\end{document}